\documentclass[12pt]{article}
\usepackage[margin=1in]{geometry}

\usepackage{natbib}
\usepackage[utf8]{inputenc} % allow utf-8 input
\usepackage[T1]{fontenc}    % use 8-bit T1 fonts
\usepackage{url}            % simple URL typesetting
\usepackage{booktabs}       % professional-quality tables
\usepackage{amsfonts}       % blackboard math symbols
\usepackage{nicefrac}       % compact symbols for 1/2, etc.
\usepackage{microtype}      % microtypography

\usepackage{morenotations,rotating}
\usepackage{color,colortbl,soul,upgreek}
\usepackage{arydshln}
\usepackage{wrapfig}
\usepackage{soul}
\usepackage{amsmath}
\usepackage{empheq}
\usepackage{mdframed,enumerate}

\usepackage{graphicx,subfigure}
\usepackage{caption}

\usepackage{authblk}
\definecolor{Gray}{gray}{0.9}

\usepackage{tikz, pgfplots}
\usetikzlibrary{intersections, calc, positioning, decorations.pathreplacing}
\pgfplotsset{compat=1.8, xlabel style={anchor=west, align=center}, ylabel style={anchor=south, align=center}, samples=200, ymin=0, ymax=1, width=3cm, height=3cm, axis lines=middle, xticklabel style={/pgf/number format/.cd,frac,frac TeX=\frac}, yticklabel style={/pgf/number format/.cd,frac,frac TeX=\frac}, xtick=\empty,ytick=\empty, no markers, cycle list={{black,solid}}, samples=200} 

\usepackage{enumitem}

\usepackage{cleveref}

\title{\papertitle}
\author{
  Richard Nock$^{\dagger}$ \qquad Aditya Krishna Menon$^{\ddagger}$ \\
 $^{\dagger}$Data61 $\&$ the Australian National University,
 $^\ddagger$Google Research\\
{\normalsize richard.nock@data61.csiro.au, adityakmenon@google.com} \\
}
\begin{document}

\date{}

\maketitle

\begin{abstract}
% !TEX root=../nips18-adversarial-mf-1.tex

Supervised learning requires the specification of a loss function to minimise.
While the theory of admissible losses from both a computational and statistical perspective is well-developed,
these offer a panoply of different choices.
In practice, this choice is typically made in an \emph{ad hoc} manner.
In hopes of making this procedure more principled,
the problem of \emph{learning the loss function} for a downstream task (e.g., classification) has garnered recent interest.
However, works in this area have been generally empirical in nature.

In this paper, 
we revisit the {\sc SLIsotron} algorithm of~\citet{kkksEL} through a novel lens, 
derive a generalisation based on Bregman divergences,
and show how it provides a principled procedure for learning the loss.
In detail, 
we cast
{\sc SLIsotron}
as learning a loss from a family of composite square losses.
By interpreting this through the lens of \emph{proper losses},
we derive a generalisation of {\sc SLIsotron} based on Bregman divergences.
The resulting {\sc BregmanTron} algorithm
jointly learns the loss along with the classifier. 
It comes equipped with a simple guarantee of convergence for the loss it learns, and its set of possible outputs comes with a guarantee of agnostic approximability of Bayes rule.
Experiments indicate that the {\sc BregmanTron} substantially outperforms the {\sc SLIsotron}, and that the loss it learns can be minimized by other algorithms for different tasks, thereby opening the interesting problem of \textit{loss transfer} between domains.

\end{abstract}

\newpage

\section{Introduction}\label{sec-int}

Computationally efficient supervised learning essentially started with the PAC framework of \citet{vAT}, in which the goal was to learn in polynomial
time a function being able to predict a label (or class, among two possible)
for i.i.d. inputs. %, also known as observations. 
The initial
loss, whose minimization enforces the accurate prediction of
labels, was the binary \textit{zero-one loss} which returns 1 iff a mistake is made.

The zero-one loss was later progressively replaced in learning
algorithms for tractability reasons, including its non-differentiability
and the structural complexity of its minimization \citep{kvAI,ahwEM}. 
From the late nineties, a zoo of losses started to be used for
tractable machine learning (ML), the
most popular ones built from the \textit{square loss} and the \textit{logistic loss}. 
Recently, there has been a significant push to
widen even more the choice of loss;
to pick a few, see~\cite{gssLS,kkksEL,lwdUU,mmSS,nnOT,nnBD,rwCB,ssckLB,sLE,sdlskdAT}. 

Seldom do such works ground reasons for change of the loss outside of
tractability at large. It turns out that statistics and Bayes
decision theory
give a precise reason, one which has long been the object of philosophical
and formal debates~\citep{fRED}. It starts from a
simple principle:
\begin{center}
\textit{Bayes rule is optimal for the loss at hand,}
\end{center}
a property known as \textit{properness}~\citep{sEO}. %\cite{rwCB}.
Then
comes a less known subtlety: a proper loss as
commonly used for real-valued prediction, such as the square and
logistic loss, involves an implicit 
\textit{canonical link} \citep{rwCB}
function that maps
class probabilities (such as the output of Bayes rule) to
real values.
This is
exemplified by the sigmoid (inverse) link in deep learning.
%The inverse link
%is more popular in machine learning, as exemplified by the sigmoid for the
%logistic loss in deep learning.

Supervised learning in a Bayesian framework can thus
be more broadly addressed by learning a classifier \textit{and} a
link for the domain at hand, \textit{which implies learning a proper canonical
loss} with the classifier. This loss, if suitably expressed, can be used for training. This kills two birds in one shot: we get access not just to real valued predictions, but also a way to embed them into class probability
estimates via the inverse link: we directly learn to estimate Bayes
rule.

A large number of papers, especially recently, have tried to push
forward the problem of learning the loss, including \textit{e.g.}
\cite{gssLS,lwdUU,mmSS,ssckLB,sLE,sdlskdAT}, \textit{but} none of
those dealing with supervised learning alludes to properness to
ground the choice of the loss, therefore taking the risk of fitting a
loss whose (unknown) optima may fail to contain Bayes rule. To the
best of our knowledge, \citet{nnOT} is the first paper grounding the
search of the loss within properness and \citet{kkksEL} brings the
first algorithm (\slisotron) and associated
theoretical results for fitting the link --- though subject to restrictive
assumptions on Bayes rule and on the target distribution, 
%a limitation due to the use of the square loss in a proper \textit{composite} form, 
the risk to fit probabilities outside $[0,1]$,
and finally falling short of showing
convergence that would comply with the classical picture of ML, either
for training or generalization.

\textbf{Our major contribution} is a new algorithm, the \clu~(Algorithm~\ref{fig:bregmantron}),
a generalisation of the \slisotron~\citep{kkksEL} to learn
proper canonical losses.
\clu{} exploits two dual views of proper losses, 
guarantees class probability estimates in $[0,1]$, 
and uses Lipschitz constraints that can be tuned at runtime.

\textbf{Our formal contribution} includes a simple convergence guarantee for
this algorithm which alleviates all assumptions on the domain and
Bayes rule in \cite{kkksEL}. Our result shows that convergence happens
as a function of the discrepancy between our estimate and the true
value of of the mean operator -- a sufficient statistic for
the class \citep{pnrcAN}. As the discrepancy converges to zero, the estimated (link, classifier) by
the \clu~converges to a stable output. To come to this result, we
pass through an intermediate step in which we show a particular
explicit form for any differentiable proper composite loss, of a
Bregman divergence \citep{bTR}.
All proofs are given in an appendix, denoted \supplement.

\section{Definitions and notations}
\label{sec:defs}

The following shorthands are used: $[n]
\defeq \{1, 2, ..., n\}$ for $n \in \mathbb{N}_*$, 
for $z \geq 0, a \leq  b \in \mathbb{R}$,
denote
$z \cdot [a, b]
\defeq [za, zb]$  and $z + [a, b]
\defeq [z+a, z+b]$. We also let
$\overline{\mathbb{R}} \defeq [-\infty, \infty]$.
In (batch) supervised
learning, one is given a training set of $m$ examples
$S \defeq \{({\ve{x}}_i, y^*_i), i \in [m]\}$, where ${\ve{x}}_i
\in {\mathcal{X}}$ is an observation
(${\mathcal{X}}$ is called the domain: often,  ${\mathcal{X}}\subseteq {\mathbb{R}}^d$) and $y^*_i
\in \mathcal{Y} \defeq \{-1,1\}$ is a label, or class. The objective
is to learn a \textit{classifier} $h
: \mathcal{X} \rightarrow \mathbb{R}$ which belongs to a given
set $\mathcal{H}$. The goodness of fit of some $h$ on ${S}$ is
evaluated by a \textit{loss}. 

\noindent $\triangleright$ \textbf{Losses}: 
A loss for binary class probability
estimation~\cite{bssLF} is some $\ell : \mathcal{Y} \times [0,1]
\rightarrow \overline{\mathbb{R}}$ whose expression can be split according to
\textit{partial} losses $\partialloss{1}, \partialloss{-1}$,
\begin{eqnarray}
\ell(y^*,u) & \defeq & \iver{y^*=1}\cdot \partialloss{1}(u) +
                     \iver{y^*=-1}\cdot \partialloss{-1}(u), \label{eqpartialloss}
\end{eqnarray}
Its \emph{conditional Bayes risk} function
is
the best achievable loss
when labels are drawn with a particular positive base-rate,
\begin{eqnarray}
\cbr(\pi) & \defeq &
\inf_u \E_{\Y\sim \pi} \properloss(\Y, u),\label{defCBR}
\end{eqnarray} 
where %$\pi \in [0,1]$, 
$\Pr[\Y = 1] = \pi$. 
%and $\Pr[\Y = -1] \defeq 1-\pi$.
A loss for class probability
estimation $\ell$ is \textit{\textbf{proper}} iff Bayes prediction
locally achieves the minimum everywhere: $ \cbr(\pi) = \E_{\Y} \properloss(\Y, \pi), \forall \pi
\in [0,1]$,
and strictly proper if Bayes is the unique minimum.
Fitting a prediction $h(\ve{x}) \in \mathbb{R}$ into some $u \in [0,1]$
as required in \eqref{eqpartialloss} is done via a 
\emph{link function}.

\noindent $\triangleright$ \textbf{Links, composite and canonical
  proper losses}. A link $\psi : [0,1] \rightarrow \mathbb{R}$ allows
to connect real valued prediction and class probability estimation. A
loss can be augmented with a \emph{link} to account for real valued prediction, $\ell_\psi( y^*, z ) \defeq \ell( y^*,
\psi^{-1}( z ) )$ with $z\in \mathbb{R}$ \citep{rwCB}. There exists a
particular link uniquely defined\footnote{Up to multiplication or addition by a scalar
\citep{bssLF}.} for any proper differentiable loss, the
\textit{canonical link}, as: $\psi \defeq - \cbr'$ \citep[Section
6.1]{rwCB}. We note that the differentiability condition can be
removed \citep[Footnote 6]{rwCB}.
As an example, for log-loss we find
the link
$\psi( u ) = \log \frac{u}{1 - u}$,
with inverse the well-known sigmoid $\psi^{-1}( z ) = ({1 +
  e^{-z}})^{-1}$. A canonical
proper loss is a proper loss using the canonical link. 

\noindent $\triangleright$ \textbf{Convex surrogates}.  
 When the loss
is proper canonical and symmetric ($\partialloss{1}(u) = \partialloss{-1}(1-u),
\forall u \in (0,1)$), it was shown in \citet{nnOT,nnBD} that 
there exists a convenient dual formulation amenable to direct
minimization with real valued classifiers: a \emph{convex surrogate} loss
% Such a surrogate is defined as: 
\begin{eqnarray}
F_\ell(z) & \defeq & (-\cbr)^\star(-z),\label{defCS}
\end{eqnarray}
where $\star$ denotes the Legendre conjugate
of $F$, $F^\star(z)\defeq \sup_{z' \in \mathrm{dom}(F)}\{zz' - F(z')\}$
\citep{bvCO}. For simplicity, we just call $F_\ell$ the convex
surrogate of $\ell$. The logistic, square and Matsushita losses are
all surrogates of proper canonical and symmetric losses. 
Such
functions are called surrogates since they all 
%turn out to 
define convenient upperbounds of the 0/1 loss. %, itself proper but not strictly proper. 
Any proper canonical
and symmetric loss has $\properloss( y^*, z ) \propto F_\ell(y^*z)$ so
both dual forms are equivalent in terms of minimization \cite{nnOT,nnBD}.

\noindent $\triangleright$ \textbf{Learning}. Given a sample $S$, we
learn $h$ by the empirical minimization of a proper loss on $S$
that we denote $\ell_\psi (S, h) \defeq \expect_{i}[\ell( y^*_i,
\psi^{-1}( h(\ve{x}_i) ) )]$. We insist on the fact that minimizing
any such loss does not just give access to a real valued predictor
$h$: it \emph{also} gives access to a %maximum likelihood 
class probability
estimator given the loss \citep[Section 5]{nnOT}, \cite{nwLO}, 
\begin{eqnarray}
\Pr[\Y = 1 | \ve{x}; h, \psi] & \defeq & \psi^{-1}(h(\ve{x})),
\end{eqnarray}
so in the Bayesian framework, supervised learning can also encompass
learning the link $\psi$ of the loss as well. 
\textit{If the loss is proper
canonical, learning the link implies learning the loss}. 
As usual, we assume 
$S$ sampled i.i.d. according to an unknown but fixed
 $\mathcal{D}$ and let $\ell_. (\mathcal{D}, h) \defeq
\expect_{S \sim \mathcal{D}} [\ell_. (S, h)]$. 

\section{Related work}\label{sec-rel}

% In light of this, we state
% We begin by formalising the 
Our
problem of interest is
learning not only a classifier, but also a \emph{loss function} itself.
A minimal requirement for the loss to be useful is that it is proper,
i.e.,
it preserves the Bayes classification rule.
Constraining our loss to this set ensures standard guarantees on the classification performance using this loss,
e.g.,
using surrogate regret bounds.

Evidently, when choosing amongst losses, we must have a well-defined objective.
We now 
reinterpret an algorithm of~\citet{kkksEL}
as providing such an objective.

%review an existing algorithm of~\citet{kkksEL}, and interpret it as precisely providing such an objective.

\noindent $\triangleright$ \textbf{The \slisotron{} algorithm}.
\citet{kkksEL} considered the problem of learning 
a class-probability model of the form
$ \Pr( \Y = 1 \mid \ve{x} ) = u( \ve{w}_*^\top \ve{x} ) $
where $u( \cdot )$ is a 1-Lipschitz, non-decreasing function,
and $\ve{w}_* \in \mathbb{R}^d$ is a fixed vector.
They proposed \slisotron{},
an 
iterative algorithm that alternates between gradient steps to estimate $\ve{w}_*$,
and 
nonparametric \emph{isotonic regression} steps to estimate $u$.
\slisotron{} provably bounds the expected \emph{square loss},
i.e.,
\begin{eqnarray}
\sqloss_\psi (S, h) & = & \expect_{\ve{x}\sim
                          S}\left[\expect_{y^* \sim S}[(y -
                          \psi^{-1}(h(\ve{x})))^2|
                          \ve{x}]\right], \label{properSQ}
\end{eqnarray}
where $h(\ve{x}) = \ve{w}^\top \ve{x}$ is a linear scorer and 
%$y \defeq 0$ iff $y^* = -1$ and $y \defeq 1$ otherwise. 
$y \defeq (y^* + 1)/2$.
The %(symmetric) 
square loss has $2
\cdot\partialsqloss{1}(u) \defeq (1-u)^2$, $2
\cdot\partialsqloss{-1}(u) \defeq u^2$ ,
and conditional Bayes risk
$2 \cdot\cbrsq(u) \defeq u(1-u)$.

Observe now that the \slisotron{} algorithm can be interpreted as follows:
we jointly learn a \emph{classifier} $h \in \mathcal{H}$ and
\emph{composite link} $\psi$ for the square loss $\ell \in
\mathcal{L}$, as
$$ \mathcal{L} \defeq \{ (y, z) \mapsto ( y - \psi^{-1}( z ) )^2 \colon \psi \text{ is 1-Lipschitz, invertible} \}. $$
That is, \slisotron{} can be interpreted as finding a classifier and
a link via all compositions of the square loss with a 1-Lipschitz, invertible function.
\citet{kkksEL} in fact do not directly analyze \eqref{properSQ} but
a \textit{lowerbound} that directly follows from Jensen's inequality:
\begin{eqnarray}
\sqloss_\psi (S, h) & = & \expect_{\ve{x}\sim
                          S}\left[\expect_{y^* \sim S}[(y -
                          \psi^{-1}(h(\ve{x})))^2|
                          \ve{x}]\right],\nonumber\\
  & \geq & \expect_{\ve{x}\sim
                          S}\left[(\expect_{y\sim S}[y|\ve{x}]-
                          \psi^{-1} \circ h(\ve{x}))^2
                          \right].\label{kkkLOSS}
\end{eqnarray}
This
does not change the problem as the slack is the expected (per
observation) variance of labels in
the sample, a constant given $S$. 
We shall return to this point in the sequel.

\citet{kkksEL} make an assumption
about Bayes rule, $\expect_{y^* \sim \mathcal{D}}[y|
\ve{x}] =
\psi_{\opt}^{-1}(\ve{w}_{\opt}^\top
\ve{x})$ with $\psi_{\opt}^{-1}$ Lipschitz and $\|\ve{w}_{\opt}\| \leq
R$.
Under such an assumption, it is shown that \textit{there exists} an iteration
$t = O((Rm/d)^{1/3})$ of the \slisotron{} with $\max\{\tilde{\sqloss}_\psi (S, h),
\tilde{\sqloss}_\psi (\mathcal{D}, h)\}\leq \tilde{O}((dR^2/m)^{1/3})$
with high probability. Nothing is guaranteed outside this unknown "hitting"
point, which we partially attribute to the lack of convergence
results on
training. Another potential downside from the Bayesian standpoint is
that the estimates learned are not guaranteed to be in $[0,1]$ by the
isotonic regression as modeled. 
%they can be negative or larger than 1.

%
\begin{figure*}[t]
\begin{mdframed}[style=MyFrame]
  {\vspace{-0.61cm}\colorbox{gray!20}{\fbox{\textbf{Algorithm 0} \slisotron}}}\\

\textbf{Input}: sample $S = \{(\bm{x}_i, y_i), i = 1, 2, ..., m\}$, iterations $T \in \mathbb{N}_*$.
  
\textbf{For} $t = 0, 1, ..., T-1$

\hspace{0.5cm} [Step 1] \textbf{If} $t=0$ \textbf{Then}
$\ve{w}_{t+1} = \ve{w}_1 = \ve{0}$ 
  \textbf{Else} fit $\ve{w}_{t+1}$ using 
  
%   a
%   gradient step with learning rate $\upeta = 1$ towards:
  
  \vspace{-0.55cm}
  
  \begin{eqnarray}
% \ve{w}^* & \defeq & \arg\min_{\ve{w}}
%                     \expect_{{S}}[D_{U_{t}}(\ve{w}^\top\ve{x}\|u^{-1}_{t}(y))],\label{pbregW}\mbox{
%                     \hspace{1.2cm}//proper canonical fitting of
%                     $\ve{w}_{t+1}$ given $\hve{y}_{t}, u_t$}
    \ve{w}_{t+1} & = & \ve{w}_{t} - \frac{1}{m} \sum_{i = 1}^m (u_{t}( \ve{x}_i )  - y_i) \cdot \ve{x}_i
  \end{eqnarray}
  
  \vspace{-0.35cm}
  
  \hspace{0.5cm} [Step 2] order indexes
  in $S$ so that $\ve{w}_{t+1}^\top\ve{x}_{i+1} \geq
  \ve{w}_{t+1}^\top\ve{x}_i, \forall i\in [m-1]$; 

  \hspace{0.5cm} [Step 3] let $\ve{z}_{t+1} \defeq \ve{w}_{t+1}^\top \ve{x}_{i + 1}$

\hspace{0.5cm} [Step 4] fit next link

  \vspace{-0.75cm}
  
\begin{eqnarray}
u_{t+1} & \leftarrow & \sc{\rm IsotonicReg}(\hve{z}_{t+1}, S); \mbox{
                    \hspace{1.1cm}//fitting of
                         $u_{t+1}$ given $\ve{z}_{t+1}$}
  \end{eqnarray}

  \vspace{-0.35cm}
  
  \textbf{Output}: $u_T, \ve{w}_T$.
\end{mdframed}
\caption{The \slisotron~algorithm of~\citet{kkksEL}.}
\label{fig:slisotron}
\end{figure*}

\noindent $\triangleright$ \textbf{Learning the loss}.
Over the last decade, the problem of learning the loss has seen a
considerable push for a variety of reasons: \citet{sdlskdAT} introduced a family of tunable losses, a subset of
which being proper, aimed at increasing robustness in
classification. \citet{mmSS} formulated the generalized linear model
using Bregman divergences, though no relationship with proper losses
is made, the loss function used integrates several regularizers
breaking properness, the
formal results rely on several quite restrictive assumptions and the
guarantees can be quite loose if the true composite link comes from a
loss that is not strongly convex "enough". In \citet{sLE}, the problem
studied is in fact learning the regularized part of the logistic
loss, with no approximation guarantee. In \citet{gssLS}, the goal is to learn a loss defined by a
neural network, without reference to proper losses and no
approximation guarantee. Such a line of work also appears in a
slightly different form in \citet{lwdUU}. In \citet{ssckLB}, the loss considered is
mainly used for metric learning, but integrates Bregman divergences. No mention of properness is
made. Perhaps the most restrictive part of the approach is that it
fits piecewise linear divergences, which are therefore not differentiable nor
strictly convex. 

Interestingly, none of these recent references
alludes to properness to constrain the choice of the loss. 
Only the
modelling of \citet{mmSS} can be related to properness via Theorem
\ref{thLBREG} proven below.
The problem of learning the loss was introduced
as loss \textit{tuning} in \citet{nnOT} (see also \citet{rwCB}). Though a general boosting
result was shown for any tuned loss following a particular
construction on its Bayes risk, it was restricted to losses defined from a
convex combination of a basis set and no insight on improved convergence rates was given.
\section{Learning proper canonical losses}\label{sec-clu} 

\begin{figure*}[t]
\begin{mdframed}[style=MyFrame]
  {\vspace{-0.61cm}\colorbox{gray!20}{\fbox{\textbf{Algorithm 1} \clu}}}\\

\textbf{Input}: sample $S = \{(\bm{x}_i, y_i), i
  = 1, 2, ..., m\}$, iterations $T \in \mathbb{N}_*$, 
  parameters $a, b \in \mathbb{R}_+$.
  
Initialize $u_0(z) \defeq 0\vee (1\wedge (az+b)) $;
  
\textbf{For} $t = 0, 1, ..., T-1$

\hspace{0.5cm} [Step 1] \textbf{If} $t=0$ \textbf{Then}
$\ve{w}_{t+1} = \ve{w}_1 = \ve{0}$ 
  \textbf{Else} fit $\ve{w}_{t+1}$ using a
  gradient step towards:
  
  \vspace{-0.55cm} 
  
  \begin{eqnarray}
\ve{w}^* \hspace{-0.3cm}& \defeq & \hspace{-0.3cm}\arg\min_{\ve{w}}
                    \expect_{{S}}[D_{U_{t}}(\ve{w}^\top\ve{x}\|u^{-1}_{t}(y))];\label{pbregW}\mbox{
                    \hspace{-0.15cm}//proper canonical fitting of
                    $\ve{w}_{t+1}$ given $\hve{y}_{t}, u_t$}
  \end{eqnarray}

  \vspace{-0.25cm}
  
  \hspace{0.5cm} [Step 2] order indexes
  in $S$ so that $\ve{w}_{t+1}^\top\ve{x}_{i+1} \geq
  \ve{w}_{t+1}^\top\ve{x}_i, \forall i\in [m-1]$; 

  \hspace{0.5cm} [Step 3] fit $\hve{y}_{t+1}$ by solving for
  global optimum ($n_t, N_t$ chosen so that $0 < n_t \leq N_t $):
  
  \vspace{-0.55cm}
  
\begin{eqnarray}
\hve{y}_{t+1} \hspace{-0.1cm}& \defeq & \hspace{-0.1cm}\arg\min_{\hat{\ve{y}}}
\expect_{{S}}[D_{U^\star_{t}}(\hat{y}\|\hve{y}_{t})];\mbox{
                    \hspace{1.1cm}//proper composite fitting of
                         $\hve{y}_{t+1}$ given $\ve{w}_{t+1}, u_t$}\nonumber\\
& & \mbox{ s.t. }  \left\{\begin{array}{l}
\hat{y}_{i+1} - \hat{y}_{i}
\in [n_t\cdot (\ve{w}_{t+1}^\top(\ve{x}_{i+1} - \ve{x}_i)), N_t\cdot
(\ve{w}_{t+1}^\top(\ve{x}_{i+1} - \ve{x}_i))]\:\:, \forall i \in
                        [m]\\
  \hat{y}_1 \geq 0, \hat{y}_m \leq 1
% \hat{y}_1 \geq 0\\
% \hat{y}_m \leq 1
\end{array}\right.\label{pbregU}.
\end{eqnarray}

\hspace{0.5cm} [Step 4] fit next inverse link

  \vspace{-0.75cm}
  
\begin{eqnarray}
u_{t+1} & \leftarrow & \fit(\hve{y}_{t+1}, \ve{w}_{t+1}, S); \mbox{
                    \hspace{2.8cm}//fitting of
                         $u_{t+1}$ given $\ve{w}_{t+1}, \hve{y}_{t+1}$}
  \end{eqnarray}

  \vspace{-0.35cm} 
  
  \textbf{Output}: $u_T, \ve{w}_T$.
\end{mdframed}
\caption{The \clu~algorithm.}
\label{fig:bregmantron}
\end{figure*}

We now present {\clu}, our algorithm to learn proper canonical losses
by learning a link function.
We proceed in two steps. We first show an explicit form to proper differentiable
composite losses and then provide our approach, the \clu, to fit such losses.

\noindent $\triangleright$ \textbf{Every proper differentiable
composite loss is Bregman} Let $F:\mathbb{R} \rightarrow \mathbb{R}$ be convex differentiable. The Bregman divergence $D_F$ with
generator $F$ is:
\begin{eqnarray}
D_{F}(
  z\|z') & \defeq & F(z) - F(z') - 
  (z-z')F' (z'). \label{eqBreg}
\end{eqnarray}
Bregman divergences satisfy a number of convenient properties, many of
which are going to be used in our results. In order not to laden the
paper's body, we have summarized in \supplement~(Section
\ref{sec_FactSheetBreg}) a factsheet of all the results we use. 
% , many
% of which 
% would complement definitions used in this Section.

Our first result gives a way to move between proper composite losses and Bregman divergences.
\begin{theorem}\label{thLBREG}
Let $\ell : \mathcal{Y} \times [0,1]
\rightarrow \overline{\mathbb{R}}$ be differentiable and $\psi: [0,1]
\rightarrow \mathbb{R}$ invertible. Then $\ell$ is a
proper composite loss with link $\psi$ iff it is a Bregman divergence
with:
\begin{eqnarray}
\ell(y^*,h(\ve{x})) & = & D_{-\cbr}(y\|\link^{-1}\circ
                          h(\ve{x})),\label{eqFFF}\\
  & = & D_{(-\cbr)^\star}(\canolink \circ\link^{-1}\circ h(\ve{x}) \| \link(y))\nonumber,
\end{eqnarray}
where $\cbr$ is the conditional Bayes risk defined in \eqref{defCBR},
and we remind the correspondence  $y = (y^* + 1)/2$.
\end{theorem}
Though similar forms of this Theorem have been proven in the past in \citet[Theorem
2]{cmnoswMB}, \citet[Lemma 3]{nnOT}, \citet[Corollary 13]{rwCB},
\citet[Theorem 2.1]{zSB}, \citet[Section 4]{sEO}, none fit exactly to
the setting of Theorem \ref{thLBREG}, which is therefore of
independent interest and proven in \supplement,
Section \ref{proof_thLBREG}. We now remark that the approach of
\citet{kkksEL} in \eqref{kkkLOSS} in fact \textit{cannot} be
replicated for proper canonical losses in general:
because any Bregman divergence is convex
in its left parameter, we still have as in \eqref{kkkLOSS}
\begin{eqnarray}
\ell_\psi(S, h) & \geq & \expect_{\ve{x}\sim S}[D_{-\cbr}(\expect_{y\sim S}[y|\ve{x}]\|{\link}^{-1}\circ
h(\ve{x}))],\nonumber
\end{eqnarray}
\textit{but} the slack in the
generalized case can easily be found to be the expected Bregman information of
the class \citep{bmdgCW}, $\expect_{\ve{x}\sim S} [I_{-\cbr}(\Y|\ve{x})]$, with $I_{-\cbr}(\Y|\ve{x}) =
\expect_{y \sim S}[-\cbr (y)|\ve{x}] + \cbr (\expect_{y \sim
  S}[y|\ve{x}])$, which therefore depends on the loss at hand (which
in our case is learned as well). 

%\noindent $\triangleright$ \textbf{Learning class probabilities with proper canonical losses} 
\noindent $\triangleright$ \textbf{Learning proper canonical losses} 
We now focus on learning class
probabilities unrestricted to all losses having the expression in
\eqref{eqFFF}, but
with the requirement that we use the canonical link: $\psi \defeq -
\cbr'$, thereby imposing that we learn the loss as well via its
link.
Being the central piece of our algorithm, we formally
define this link alongside some key parameters that will be
learned.
\begin{definition}\label{defLINK}
A link $u$ is a strictly increasing function with $\mathrm{Im} u =
[0,1]$, for which there exists $-\infty \ll \zmin{}, \zmax{} \ll
\infty$ and $0<n\leq N$ such that (i) $u(\zmin{}) = 0,
u(\zmax{}) = 1$ and (ii) $\forall z \leq z', n(z'-z) \leq u(z') - u(z)
    \leq N(z'-z)$.
\end{definition}
Notice that relaxing $\zmin{}, \zmax{} \in \overline{\mathbb{R}}$ (the
closure of $\mathbb{R}$) and $n, N \in \overline{\mathbb{R}}_+$
would allow to encompass all invertible links, so definition
\ref{defLINK} is not restrictive but rather focuses on simply
computable links. Given the canonical link $u$, we let
\begin{eqnarray}
    U(z) & \defeq & \int_{\zmin{}}^{z} u(t) \mathrm{d}t, \label{defU}
  \end{eqnarray}
from which we obtain
the 
convex surrogate 
$F_u( z ) = U( -z )$ and 
conditional
Bayes risk
$\cbr_u (v) = - U^\star(v)$
% \begin{eqnarray}
%   %F_u(z) & = & U(-z)\nonumber\\
%   \cbr_u (v) & = & -  \int_{0}^{v} u^{-1}(t) \mathrm{d}t = - U^\star(v),\nonumber
%   \end{eqnarray}
% of the corresponding loss whose partial losses are 
for the proper loss
$\partiallossplus{u}{y^*}(c) \defeq
  D_{-\cbr_u}(y\|c), \forall y\in\{0,1\}$.
  
Inline with Theorem \ref{thLBREG}, the loss
we seek to minimize is
\begin{eqnarray}
%   \ell(y^*, h(\ve{x})) & = & D_{U^\star}(y \| u \circ h(\ve{x}))\\
%   & = & D_{U}(h(\ve{x}) \| u^{-1}(y)),
  \ell(y^*, h(\ve{x})) = D_{U^\star}(y \| u \circ h(\ve{x})) = D_{U}(h(\ve{x}) \| u^{-1}(y)),
\end{eqnarray}
with $h(\ve{x}) = \ve{w}^\top \ve{x}$ a linear classifier. 

We present
\clu{},
our
algorithm for fitting such losses in Figure
\ref{fig:bregmantron}.
\clu{} iteratively 
fits a sequence $u_0, u_1,
...$ of links and losses and $\ve{w}_0, \ve{w}_1, ...$ of
classifiers. 
% The \clu{} is presented where 
Here,
$\wedge, \vee$ are
shorthands for $\min, \max$ respectively and in Step 3, we have
dropped the iteration in the optimization problem ($\hat{y}_{i}$ denotes $\hat{y}_{ti}$). 
Notice that Steps 1 and 3 exploit the two dual views of proper losses presented in \S\ref{sec:defs}.

Before analyzing \clu, we make a few comments. First, in the initialization, we pick $\zmax{0} - \zmin{0}
\defeq \Delta > 0$ and let $a \defeq 1/\Delta, b\defeq - \zmin{0} /
\Delta$.

Second, the choice of $n_t, N_t$ is made iteration-dependent
for flexibility reasons, since in particular the gradient step does
not put an explicit limit on $\ve{w}_t$. However, there is an implicit
constraint which is to ensure $n_t\leq
1/(\ve{w}_{t+1}^\top(\ve{x}_{m}-\ve{x}_{1})) \leq N_t$ to get a non-empty feasible set for $u_{t+1}$.

Third,
\clu{} bears close similarity to \slisotron{},
but with two key points to note.
In Step 1, we perform a gradient step to minimise a divergence between the current predictions and the labels;
fixing a learning rate of $\upeta = 1$, in fact reduces to the \slisotron{} update.
Further, in Step 3, we perform fitting of $\hat{y}_{t + 1}$ based on the \emph{previous} estimates $\hat{y}_t$, rather than the observed labels themselves
as per \slisotron.
Our Step 3 can thus be seen as ``Bregman regularisation'' step, which ensures the predictions (and thus the link function) do not vary too much across iterates.
Such stability ensures asymptotic convergence, but does mean that the initial choice of link can influence the rate of this convergence.

Finally, 
with $z_{(t+1)i} \defeq
\ve{w}_{t+1}^\top\ve{x}_i$,
\fit~can be
summarized as:
\begin{itemize}
\item [\textbf{[1]}] linearly interpolate between $(z_{(t+1)i}, \hat{y}_{(t+1)i})$ and
  $(z_{(t+1)(i+1)}, \hat{y}_{(t+1)(i+1)})$, $\forall i\in
  \{2, 3, m-2\}$,
\item [\textbf{[2]}] pick $\zmin{(t+1)} \leq \ve{w}_{t+1}^\top\ve{x}_{1},
  \zmax{(t+1)}\geq \ve{w}_{t+1}^\top\ve{x}_{m}$ with:
\begin{eqnarray}
\hat{y}_{j} & \in & [l_j, r_j], j \in \{1, m\}, 
\end{eqnarray}
and linearly interpolate between 
$(\zmin{(t+1)}, 0)$ and $(z_{(t+1)1}, \hat{y}_{(t+1)1})$, 
% $\{ (\zmin{(t+1)}, 0)$, $(z_{(t+1)1}, \hat{y}_{(t+1)1}) \}$, 
and
% between 
$(z_{(t+1)m}, \hat{y}_{(t+1)m})$ and $(\zmax{(t+1)}, 1)$.
% $\{ (z_{(t+1)m}, \hat{y}_{(t+1)m})$, $(\zmax{(t+1)}, 1) \}$.
\end{itemize}
Here, $r_1
\defeq n_t\cdot (\ve{w}_{t+1}^\top\ve{x}_{1} - \zmin{(t+1)}), l_1
\defeq N_t\cdot (\ve{w}_{t+1}^\top\ve{x}_{1} - \zmin{(t+1)})$,  $r_m
\defeq n_t\cdot (\zmax{(t+1)} - \ve{w}_{t+1}^\top\ve{x}_{m}), l_m
\defeq N_t\cdot (\zmax{(t+1)} -
\ve{w}_{t+1}^\top\ve{x}_{m})$.
Figure \ref{f-link-m} presents a simple
example of \fit~in the \clu.

\begin{figure}[t]
\begin{center}
\includegraphics[trim=10bp 420bp 575bp
140bp,clip,width=0.99\linewidth]{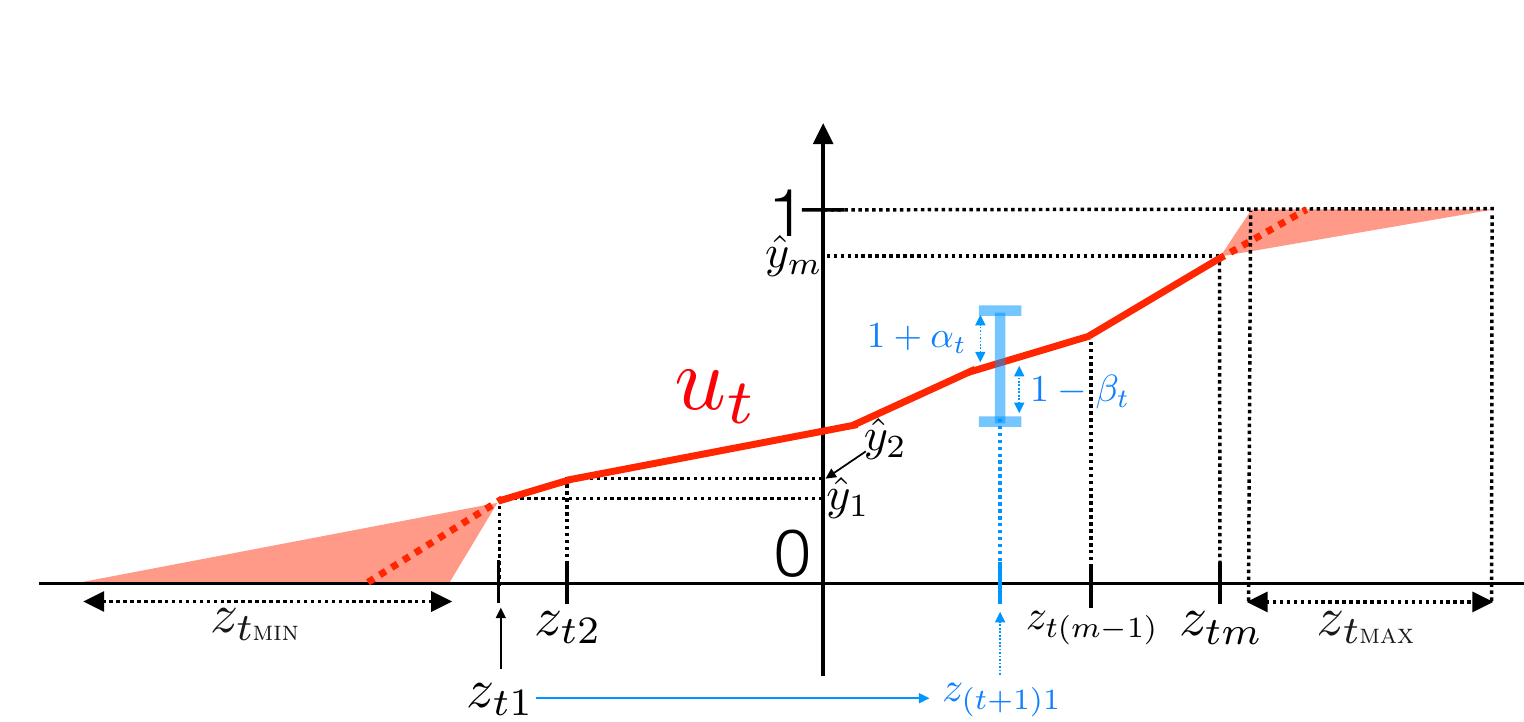}
\end{center}
\caption{Inverse link $u$ of a proper canonical loss learned (red), as
  computable in \fit~($z_{ti} \defeq
\ve{w}_{t}^\top\ve{x}_i$). In blue, we have depicted the
stability constraint in the update of the link (Definition
\ref{defAB}). Stability imposes just a constraint in the
change of inverse link for a single example. It does not impose any
constraint on the classifier update (which, in this example, is
significant for example $(\ve{x}_1, y_1)$, see text).}
  \label{f-link-m}
\end{figure}

\noindent $\triangleright$ \textbf{Analysis of {\clu}} 
We are now ready to analyze the \clu. 
Our main result shows that
provided the link does not change too much between iterations, we are
guaranteed to %have a decrease of 
decrease
%the loss at hand. 
the following loss:
% Using indices, the loss we analyse is, 
\begin{eqnarray}
  \properlossplus{t}{r}(S,\ve{w}_{t'}) & \defeq &
\expect_{{S}}[D_{U_r^\star}(y
\| u_t(\ve{w}_{t'}^\top\ve{x}))] \:\:,
\end{eqnarray}
for $r,t,t'=1, 2, ...$,
which gives \eqref{eqFFF} for $-\cbr \defeq U_r^\star$, $u_t\defeq
\psi^{-1}$ and $h(\ve{x})\defeq \ve{w}_{t'}^\top\ve{x}$. %Remark that
We do not impose $r=t$, as 
%it turns out 
our algorithm incorporates a
step of proper composite fitting of the next link given the current loss. 

We formalise the stability of the link below.
%We formalize this below. %the
%change allowed in the following definition.
\begin{definition}\label{defAB}
Let $\alpha_t, \beta_t \geq 0$. 
\clu~is $(\alpha_t,
\beta_t)$-stable at iteration $t$ iff the solution $\hve{y}_{t+1}$ in
Step 3 satisfies $\hat{y}_1 \in u_{t}(\ve{w}_{t+1}^\top\ve{x}_1)\cdot [1-\beta_t, 1+\alpha_t]$.
\end{definition}
Since $\hat{y}_1 \defeq u_{t+1}(\ve{w}_{t+1}^\top\ve{x}_1)$ from \fit,
it comes that stability requires a bounded local change in
inverse link for \textit{a single example} (but implies a bounded
change for all via the Lipschitz constraints in Step 3; this
is explained in the proof of the main Theorem of this Section). We
define the following \textit{mean operators} \citep{pnrcAN}:
$\hve{\mu}_{y} \defeq \expect_{{\mathcal{S}}}[y\cdot \ve{x}]$ (sample), $\hve{\mu}_{t} \defeq
\expect_{{\mathcal{S}}}[\hat{y}_{t} \cdot \ve{x}], \forall t\geq 1$ (estimated),
where $\hat{y}_t$ is defined in the \clu. We also let $p^*_t \defeq
\max\{\expect_{{S}}[y],    \expect_{{S}}[u_t(\ve{w}_{t+1}^\top\ve{x})]\} \:\:
                 \in [0,1]$ denote the $\max$ estimated $\hat{\Pr}( \Y = 1 )$ using both our model and
                 the sample $S$. Assume $p_t^* > 0$ as otherwise the problem is
                 trivial. Finally, $X \defeq \max_i \|\ve{x}_i\|_2$ (we consider the $L_2$ norm for simplicity; our result holds for any norm on $\mathcal{X}$).
\begin{definition}
\clu~is said to be in the $\delta_t$-regime at iteration $t$, for some
$\delta_t > 0$ iff:
\begin{eqnarray}
\|\hve{\mu}_{y} -
\hve{\mu}_{t}\|_2 & \geq & 2\sqrt{p^*_t \delta_t} X, \forall t.\label{condMUXMF}
\end{eqnarray}
\end{definition}
To simplify the statement of our Theorem, we let $f(z) \defeq
z/(1+z)$, which satisfies $f(\overline{\mathbb{R}}_+) = [0,1]$.
\begin{theorem}\label{thCLU}
Suppose that \clu~is in the $\delta_t$-regime at iteration
$t$, and the following holds:
\begin{itemize}
\item in Step 1, the learning rate
\begin{eqnarray*}
\upeta_t & = & \frac{1-\gamma_t}{2 N_t X^2} \cdot \left( 1-
      \frac{f(\delta_t)(1+f(\delta_t))  p^*_t
X}{\|\hve{\mu}_{y} -
\hve{\mu}_{t}\|_2} \right),
\end{eqnarray*}
for some user-fixed $\gamma_t \in \left[ 0, \sqrt{f\circ f(\delta_t)/2}
\right]$;
\item in Step 3, $N_t,n_t$ satisfy $\nicefrac{N_t}{n_t}, \nicefrac{N_{t-1}}{n_t} \leq 1 + f(\delta_t)$.
\end{itemize}
Then if the \clu~is $(f(\delta_t), f(\delta_t))$-stable, then:
\begin{eqnarray}
\properlossplus{t+1}{t+1}(S,\ve{w}_{t+1}) & \leq &
                                                   \properlossplus{t}{t}(S,\ve{w}_{t})- \frac{p^*_t f(\delta_t)}{n_t}.\label{boundERR1}
\end{eqnarray}
\end{theorem}
(proof in \supplement,
Section \ref{proof_thCLU}) %In synthetic form, 
Explicitly, it can be shown that the
learning rate at iteration $t$ lies in the following interval:
\begin{eqnarray}
\upeta & \in & \frac{1-\gamma_t}{2 N_t X^2} \cdot \left( 1 - \frac{\sqrt{\delta_t p^*_t} (2+\delta_t)}{2(1+\delta_t)^2} \cdot \left[\sqrt{\delta_t p^*_t}, 1\right]\right),\nonumber
\end{eqnarray}
Theorem \ref{thCLU} essentially says that as long as $\hve{\mu}_{y} \neq
\hve{\mu}_{t}$, we can hope to get better results. This is no
surprise: the gradient step in Step 1 of \clu~is proportional
to $\hve{\mu}_{y} -
\hve{\mu}_{t}$. 

The conditions on Steps 1 and 3 are easily enforceable at
any step of the algorithm, so the Theorem essentially says that whenever
the link does not change too much between iterations, we are
guaranteed a decrease in the loss and therefore a better fit of the
class probabilities. %involving both the link and the classifier. 
Stability is the only assumption made: unlike \citet{kkksEL}, no assumptions are made
about Bayes rule or the distribution
$\mathcal{D}$, and no constraints are put on
the classifier $\ve{w}$.

We can also choose to \textit{enforce} stability in the
update of $u$ in Step 3. Interestingly, while this restricts the choice of links (at
least when $\|\hve{\mu}_{y} -
\hve{\mu}_{t}\|_2$ is small), this \textit{guarantees} the bound in
\eqref{boundERR1} at no additional cost or assumption. 
\begin{corollary}\label{corCLU}
Suppose \clu~is run with so that in Step 3,
constraint $\hat{y}_{1}\geq 0$ is replaced by
\begin{eqnarray}
\hat{y}_1 & \in & u_{t}(\ve{w}_{t+1}^\top\ve{x}_1)\cdot [1-\beta_t,
                  1+\alpha_t],
\end{eqnarray} 
and all other constraints are kept the same. Suppose \clu~is in the $\delta_t$-regime at iteration
$t$, parameters $\upeta_t, N_t, n_t$ are fixed as in Theorem
\ref{thCLU} and furthermore $\alpha_t, \beta_t \in [0,
f(\delta_t)]$. Then \eqref{boundERR1} holds.
\end{corollary}
(proof in \supplement,
Section \ref{proof_corCLU}) 
There exists an alternative reading of Corollary \ref{corCLU}: there
exists a way to fix the key parameters ($\upeta_t, n_t, N_t, \alpha_t,
\beta_t$) at each iteration such that a decrease of the loss is
guaranteed \textit{if} our current estimate
of the sample mean operator, $\hve{\mu}_{t}$, is not good enough.

\section{Discussion}\label{sec-dis}

Bregman divergences have had a rich history outside of convex
optimisation, where they were introduced \citep{bTR}. They are the
canonical distortions on the manifold of parameters of exponential
families in information
geometry \citep{anMO}, they have been introduced in normative economics
in several contexts \citep{mnID,sTCO}. In machine learning, their
re-discovery was grounded in their representation and algorithmic
properties, starting with the work of M. Warmuth and collaborators
\citep{hkwWC,hwTT}, later linked back to exponential families
\citep{awRLBEF}, and then axiomatized in unsupervised learning
\citep{bmdgCW,bgwOT}, and then in supervised learning (See Section \ref{sec-clu}).

We do not investigate in this paper the generalization abilities of
\clu. It either follows from classical uniform convergence bounds
applied to the convex surrogate of any proper canonical loss
(which is always 1-Lipschitz), for which refer to \citet[Section
2]{bmRA} and references therein for
available tools, or it would justify a paper of its own if we want to
directly investigate the approximation of Bayes rule, a problem that also
entails the prospective problems presented below.

The setting of the \clu~raises two questions, the first of which is
crucial for the algorithm. We make no assumption about the optimal
link, which resorts to a powerful \textit{agnostic} view of machine
learning chased in a number of works \citep{bkmTO}, but it makes much
more sense if we can prove that the link fit by \fit~belongs to a set
with reasonable approximations of the target. This set contains
piecewise affine links, which is a bit more general than Definition
\ref{defLINK} but matches the links learned by the \clu. We remove the
index notation $T$ in $u_T$ and $z_{T.}$, and consider the following $\ell_1$ restricted discrepancy,
\begin{eqnarray}
E(u, \properloss) & \defeq & \int_{z_{1}}^{z_{m}} |(-
\cbr')^{-1}(z) - u(z)|\mathrm{d}z,\label{errEFL}
\end{eqnarray}
where $\ell$ is proper canonical with invertible canonical link. It is restricted
because we do not consider set $(-\infty,z_1)\cup (z_m,+\infty)$,
whose fitting in fact does not depend on data (see Figure
\ref{f-link-m}).  Denote $\mathcal{U}_{n,N}(\ve{w}, S)$ the set of piecewise
affine links with non-$\{0,1\}$ breakout points on abscissae $z_1,
z_2, ..., z_m$ ($\ve{w}^\top \ve{x}_i\defeq z_i<
z_{i+1} \defeq \ve{w}^\top \ve{x}_{i+1}, \forall i\in \{0, 1,
..., m-1\}$, wlog), satisfying Definition \ref{defLINK}. Let $\|.\|$ be any norm on
$\mathcal{X}$ and $\|.\|_*$ its dual. For any $\varepsilon>0$, $G_{\varepsilon}(S)$ is the graph whose vertices are the observations in $S$ and an edge
links $\ve{x}, \ve{x}'\in S$ iff $\|\ve{x} - \ve{x}'\|\leq
\varepsilon$. $G$ is said 2-vertex-connected iff it is connected when any single
vertex is removed, which is a lightweight condition that essentially
prevents the graph from being constituted of two almost separate subgraphs.
\begin{lemma}\label{lemCONC}
For any $S$ of size $m$, any $\varepsilon > 0$ such that $G_{\varepsilon}(S)$ is
2-vertex-connected and any proper canonical loss
$\properloss$, $\exists n, N\ll \infty$ such that $\inf_{u \in \mathcal{U}_{n,N}(\ve{w}, S)} E(u, \properloss) \leq
                                                                 2N m\varepsilon^2
                                                                 \cdot
                                                                 \|\ve{w}\|_*^2$.
\end{lemma}
Crucially, $N$ can be much smaller than the Lipschitz constant of $(-
\cbr')^{-1}$.
Lemma \ref{lemCONC} does guarantee that the set of links in which
the \clu~finds $u_T$ is powerful enough to approximate a link provided we sample enough examples to drag
$\varepsilon$ small enough while guaranteeing $G_{\varepsilon}(S)$ 2-connected. This does not require i.i.d. sampling but would require additional assumptions about $\mathcal{X}$ to
be tractable (such as boundedness), or the possibility of active
learning in $\mathcal{X}$. 
This also does not guarantee that \fit~finds a link with small $E(.,
\properloss)$, and this brings us to our second question: is it
possible that (near-)optimal solutions contain very "different"
couples $(u, \ve{w})$, for which useful notion(s) of "different" ?
This, we believe, has ties with the transferability of the loss.

\section{Experimental results}\label{sec-exp}

We present experiments illustrating:
\begin{enumerate}[label=(\alph*),itemsep=0pt,topsep=0pt]
    \item the viability of the {\clu} as an alternative to classic GLM or \slisotron~learning.
    \item the nature of the loss functions learned by the {\clu}, which are potentially \emph{asymmetric}.
    \item the potential of using the loss function learned by the {\clu} as input to some downstream learner.
\end{enumerate}

\noindent $\triangleright$ \textbf{Predictive performance of {\clu}}
We compare {\clu} as a generic binary classification method
% comparing its performance 
against the following baselines:
logistic regression,
GLMTron~\citep{kkksEL} with $u( \cdot )$ the sigmoid,
and \slisotron{}.
We also consider two variants of {\clu}:
one 
where in Step 4 we do not find the global optimum (${\clu}_{\mathrm{approx}}$),
but rather a feasible solution with minimal $\hat{y}_m$;
and another where in Step 4 we fit against the labels, rather than $\hat{y}_t$ (${\clu}_{\mathrm{label}}$).
% For each method, we use a linear model $\ve{w}^\top \ve{x}$.
% To aid numerical stability and generalisation,
% we additionally enforce that $\| \ve{w} \|_2 \leq 1$ in the update of each algorithm.
% While the theoretical guarantees for each algorithm do not require such a step,
% we found this to only be beneficial empirically.

In all experiments, we fix the following parameters for {\clu}:
we use a constant learning rate of $\eta = 1$ to perform the gradient update in Step 1,
For Step 3, we fix $n_t = 10^{-2}$ and $N_t = 1$ for all iterations.

We compare performance on two standard benchmark datasets,
the MNIST digits ({\tt mnist}) and the fashion MNIST ({\tt fmnist}).
We converted the former to a binary classification problem of the digits 0 versus 8,
and the latter of the odd versus even classes.
We also consider a synthetic dataset ({\tt synth}),
comprising 2D Gaussian class-conditionals with means $\pm (1, 1)$ and identity covariance matrix.
The Bayes-optimal solution for $\Pr( \Y = 1 \mid \X )$ can be derived in this case: it takes the form of a sigmoid, as assumed by logistic regression, composed with a linear model proportional to the expectation. In this case therefore, logistic regression works on a search space much smaller than \clu~and guaranteed to contain the optimum.

On a given dataset, we measure the predictive performance for each method via the area under the ROC curve.
This assesses the ranking quality of predictions, 
%as opposed to the fit according to, e.g., mean squared-error;
which
%this 
provides a commensurate means of comparison; %comparing different predictions;
%since 
in particular the {\clu} optimises for a bespoke loss function that can be vastly different from the square-loss.

\begin{figure*}[!t]
    \centering

    \subfigure[{\tt synth}.]{\includegraphics[scale=0.275]{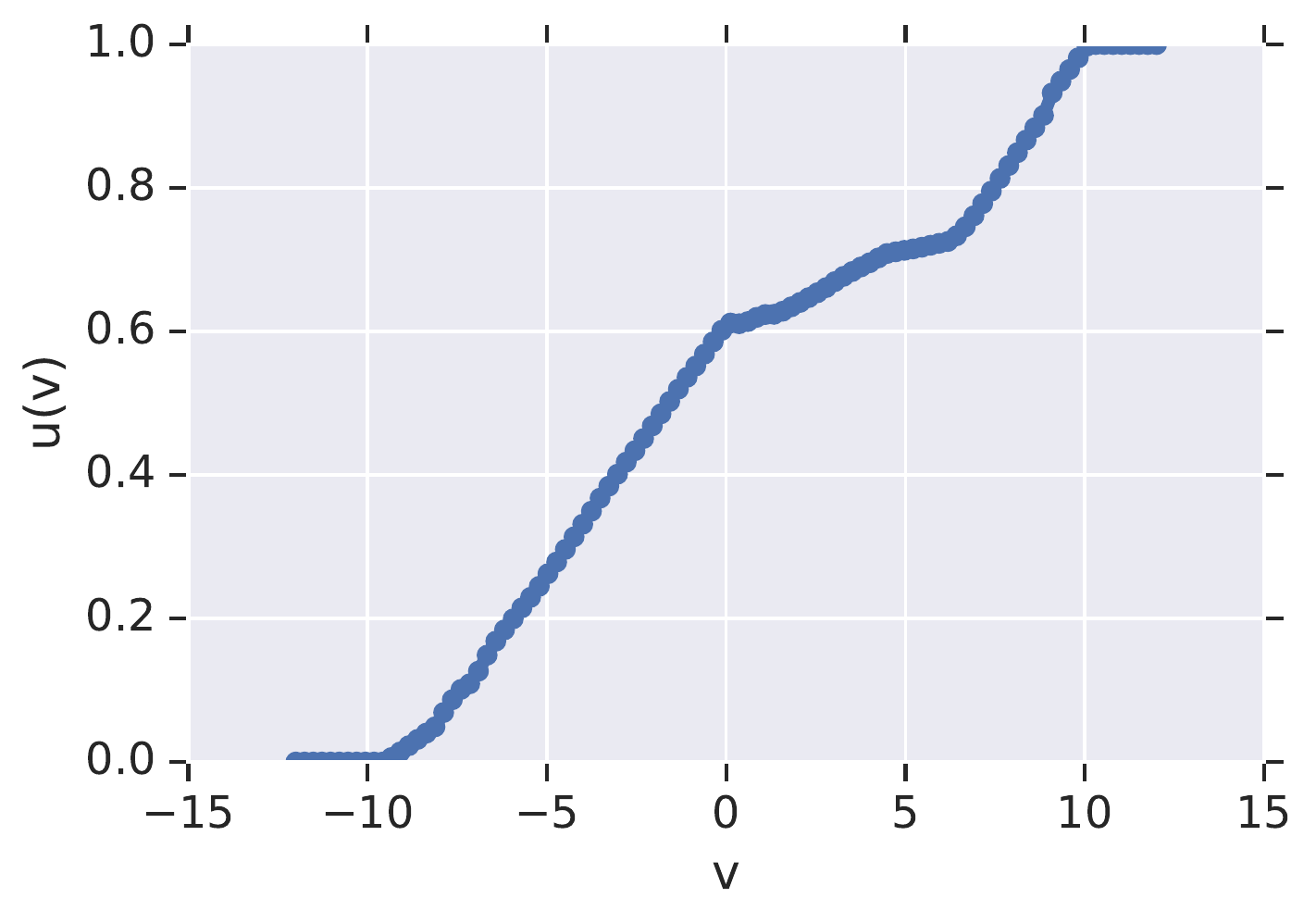}}
    \subfigure[{\tt mnist}]{\tt
    \includegraphics[scale=0.275]{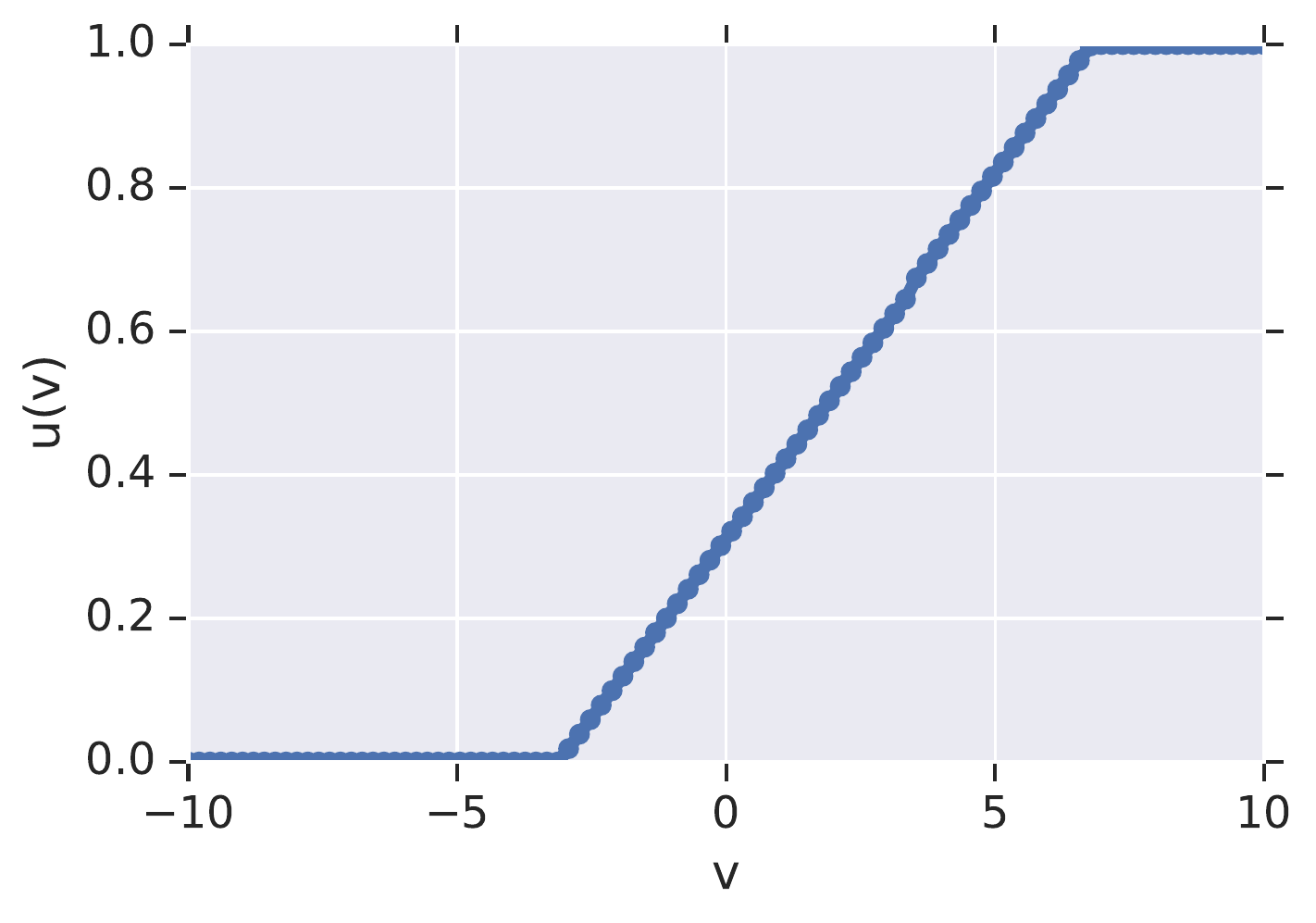}}
    \subfigure[{\tt fmnist}]{\tt
    \includegraphics[scale=0.275]{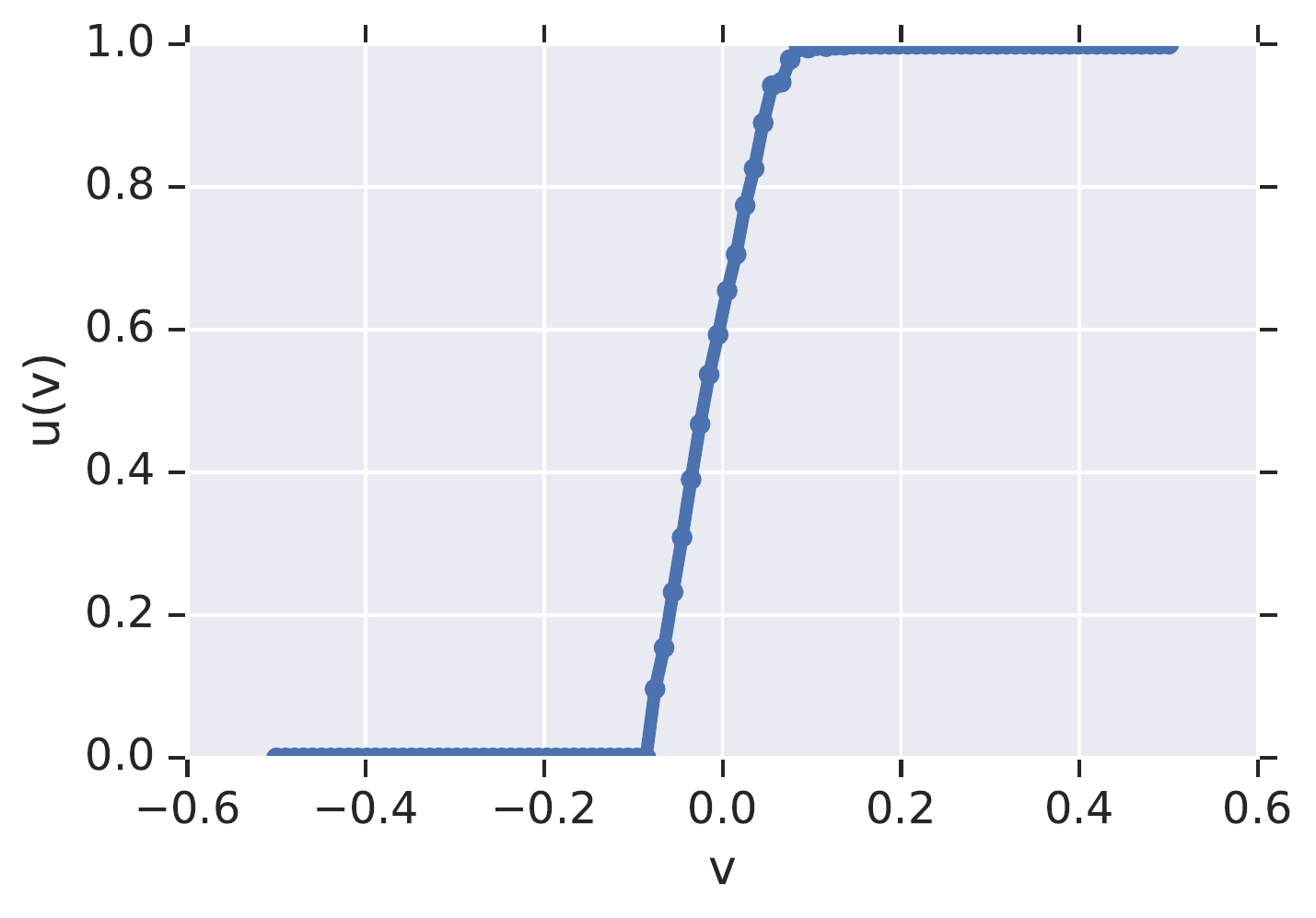}}

\vspace{-0.3cm}

    \caption{Link functions estimated by {\clu} on each dataset.
    On all datasets, the losses are seen to be (slightly) asymmetric around $\frac{1}{2}$,
    i.e.,
    $u( v ) \neq 1 - u( -v )$.
    In particular, $u( 0 ) \neq \frac{1}{2}$.}
    \label{fig:learned_links}
\end{figure*}

\begin{figure*}[!t]
    \centering

    \subfigure[{\tt synth}.]{\includegraphics[scale=0.275]{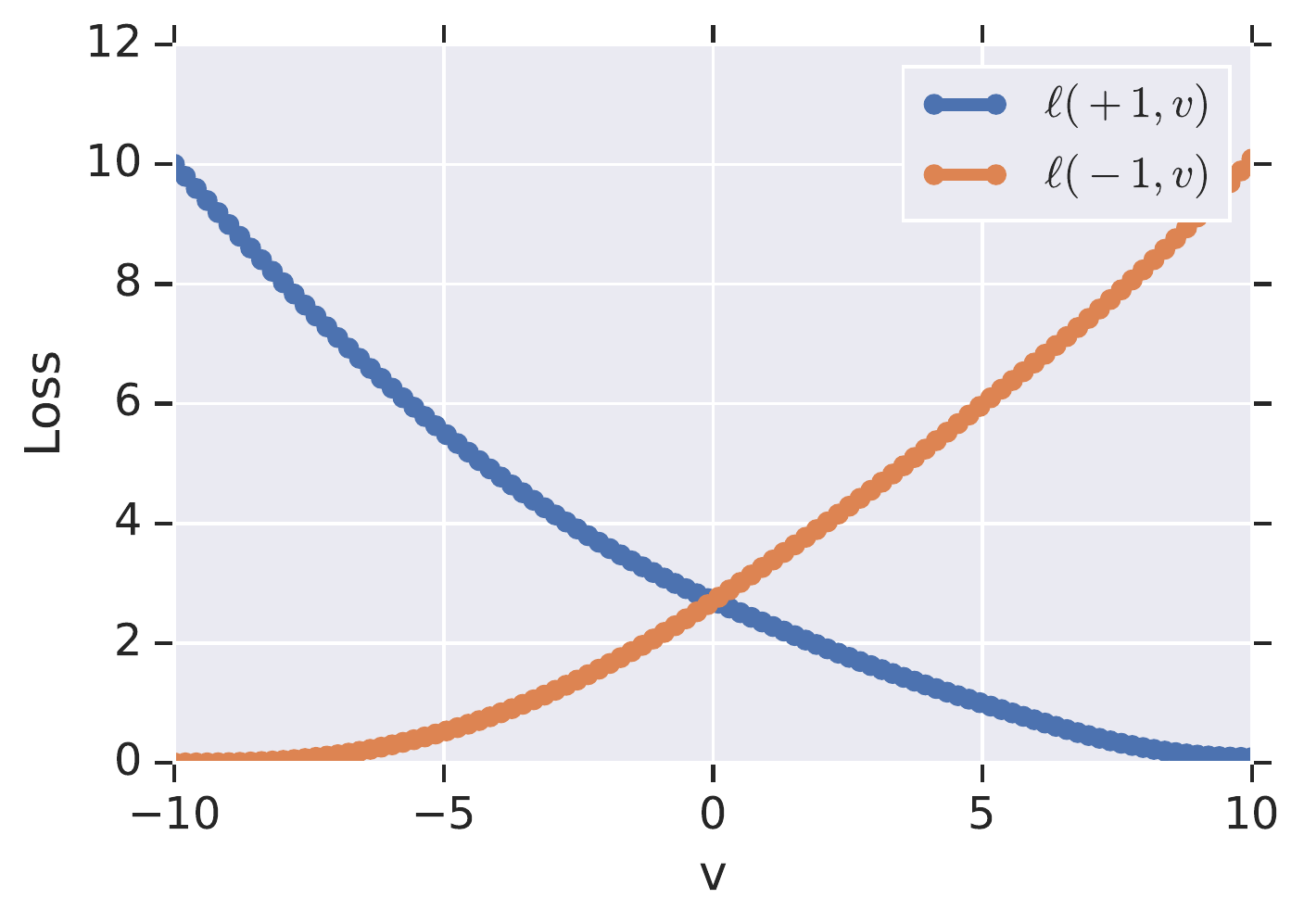}}
    \subfigure[{\tt mnist}]{\tt
    \includegraphics[scale=0.275]{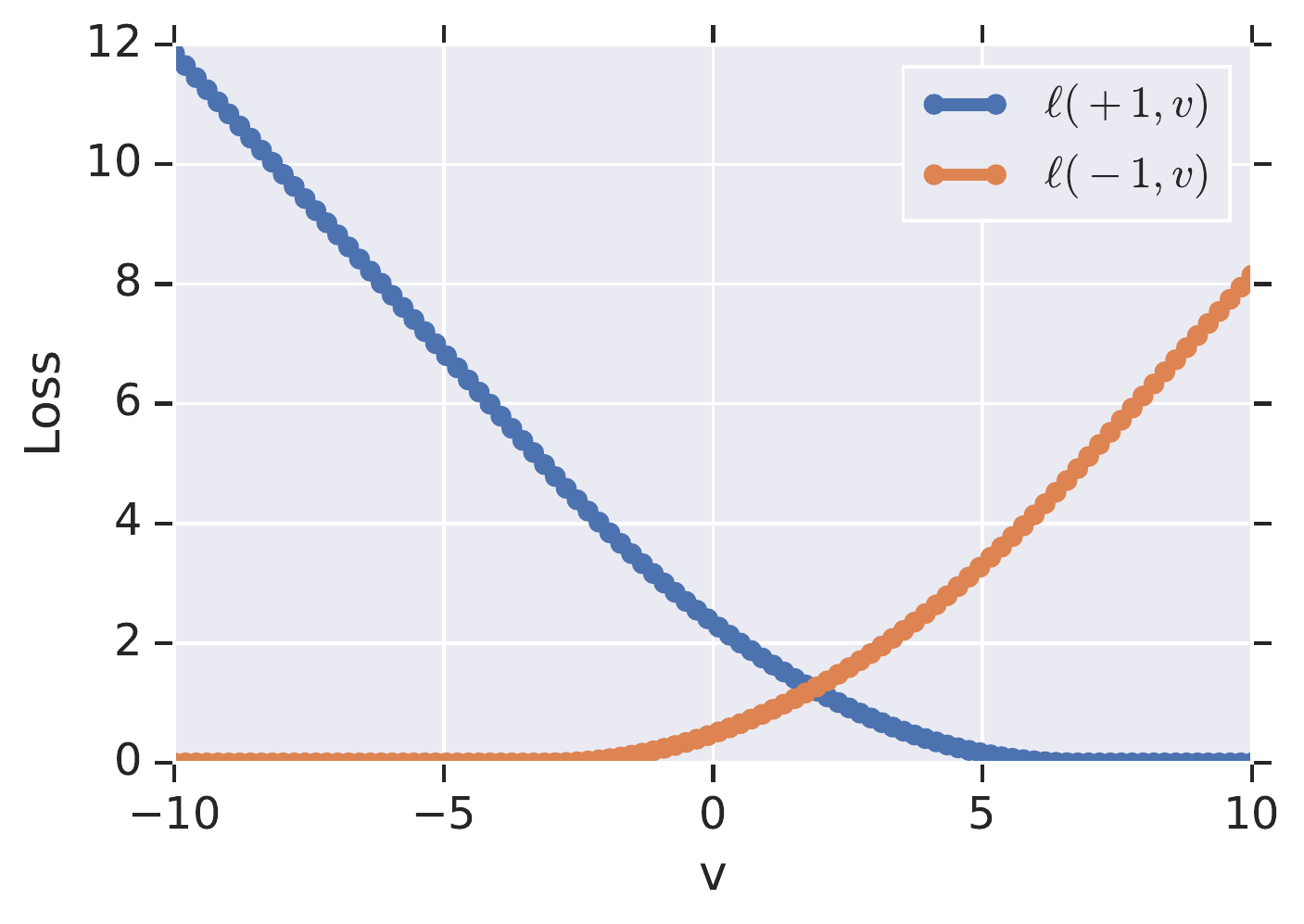}}
    \subfigure[{\tt fmnist}]{\tt
    \includegraphics[scale=0.275]{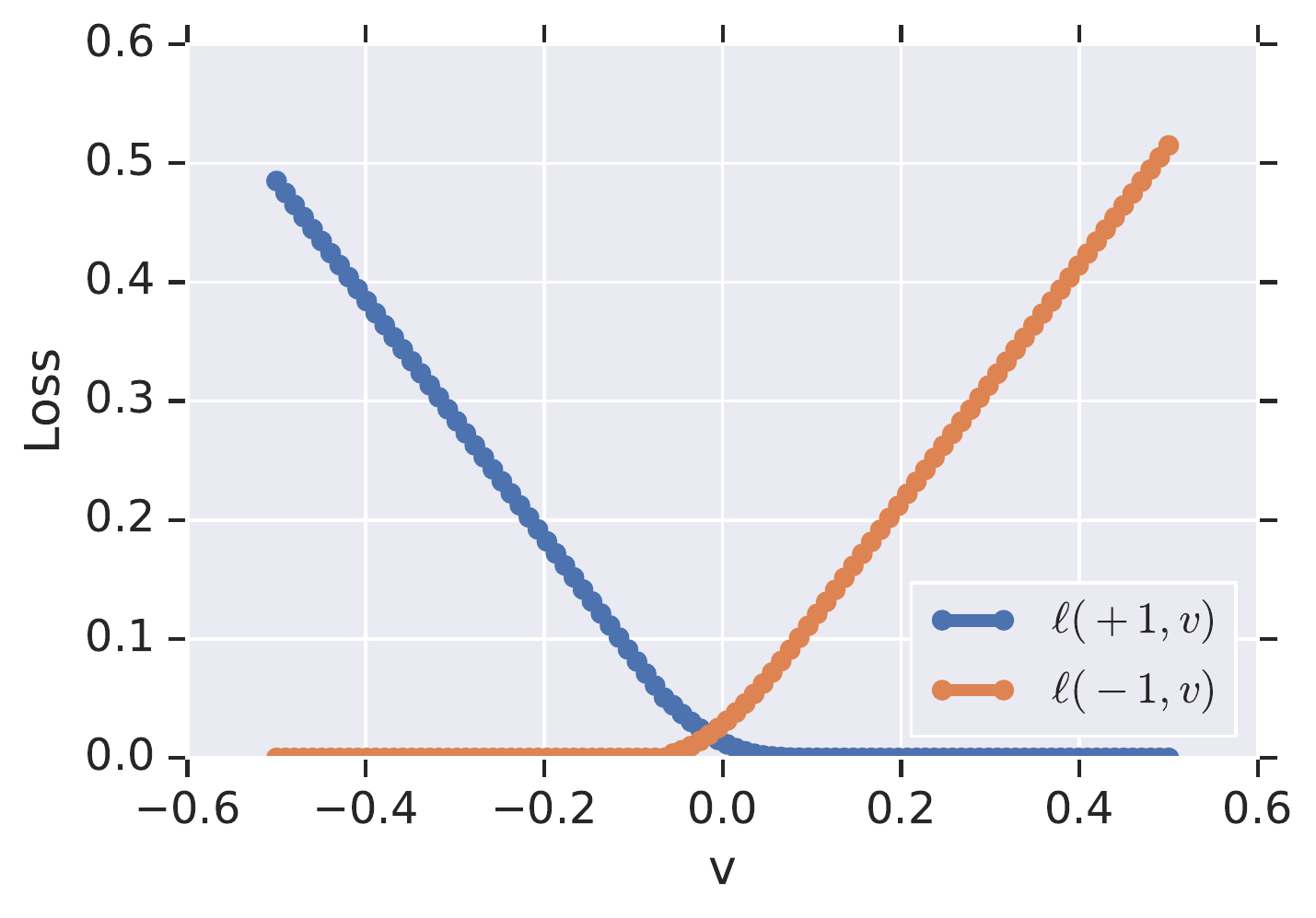}}

\vspace{-0.3cm}

    \caption{Loss functions estimated by {\clu} on each dataset.
    The losses are (slightly) asymmetric,
    and have the flavour of the square-hinge loss;
    this is a consequence of the linear interpolation used when fitting.}
    \label{fig:learned_losses}
\end{figure*}

Table~\ref{tbl:results} summarises the results.
We make three observations.
First, {\clu} is consistently competitive with the mature baseline of logistic regression.
Interestingly, this is even so on the {\tt synth} problem,
wherein logistic regression is correctly specified.
% further, the link function estimated by {\clu} here takes a different form than the sigmoid (see Figure~\ref{fig:learned_links}).
Although the difference in performance here is minor,
it does illustrate that {\clu} can infer a meaningful pair of $( u, \ve{w} )$.

Second, {\clu} 
and {\clu}$_{\mathrm{label}}$
are
generally \emph{superior} to the performance of the \slisotron.
We attribute this to the latter's reliance on an isotonic regression step to fit the links, 
as opposed to a Bregman regularisation. %step in {\clu}.

Third, while {\clu}$_{\mathrm{approx}}$ also performs reasonably, it is typically worse than the full {\clu}.
This illustrates the value of (at least approximately) solving Step 4 in the {\clu}.
Further, while {\clu}$_{\mathrm{label}}$ generally performs slightly worse than standard {\clu},
it remains competitive.
A formal analysis of this method would be of interest in future work.

\begin{table}[!t]
    \centering
    
    \resizebox{0.675\linewidth}{!}{
    \begin{tabular}{llll}
        \toprule
         & {\tt synth} & {\tt mnist} & {\tt fmnist} \\
        \midrule
        Logistic regression        & 92.2\% & 99.9\% & 98.5\% \\
        GLMTron                    & 92.2\% & 99.6\% & 98.1\% \\
        \rowcolor{lightgray}\slisotron                  & 91.6\% & 94.6\% & 90.7\% \\
        {\clu}$_{\mathrm{approx}}$ & 92.2\% & 99.3\% & 94.6\% \\
        {\clu}$_{\mathrm{label}}$  & 90.1\% & 99.6\% & 97.7\% \\
        \rowcolor{lightgray}{\clu}                     & 92.3\% & 99.7\% & 97.9\% \\
        \bottomrule
    \end{tabular}
    }
    \caption{Test set AUC of various methods on binary classification datasets. See text for details.}
    \vspace{-\baselineskip}
    \label{tbl:results}
\end{table}
 
\noindent $\triangleright$ \textbf{Illustration of learned losses}
As with the \slisotron,
a salient feature of {\clu} is the ability to automatically learn a link function.
Unlike the \slisotron, however,
the link in the {\clu} has an interpretation of corresponding to a canonical loss function.

Figure~\ref{fig:learned_links} illustrates the link functions learned by {\clu} on each dataset.
We see that these links are generally \emph{asymmetric} about $\frac{1}{2}$.
This is in contrast to standard link functions such as the sigmoid.
Recall that each link corresponds to an underlying canonical loss,
given by
% \begin{align*}
%     \ell( +1, v ) &= \int_{-\infty}^v u( z ) \, \mathrm{d} z - v \\
%     \ell( -1, v ) &= \int_{-\infty}^v u( z ) \, \mathrm{d} z.
$\ell'( y, v ) = u( v ) - y$.
Asymmetry of $u( \cdot )$ thus manifests in $\ell( +1, v ) \neq \ell( -1, -v )$.
We illustrate these implicit canonical losses in Figure~\ref{fig:learned_losses}.
As a consequence of the links not being symmetric around $\frac{1}{2}$, the losses on the positive and negative classes are not symmetric for the {\tt synth} dataset. This is unlike the \textit{theoretical} link, but the theoretical link may not be optimal at all on \textit{sampled data}. 
This, we believe, also illustrates the intriguing potential of the {\clu} to detect and exploit hidden asymmetries in the underlying data distribution.

\noindent $\triangleright$ \textbf{Transferability of the loss between domains}
Finally, we illustrate the potential of ``recycling'' the loss function implicitly learned by {\clu} for some other task.
We take the {\tt fmnist} dataset, and first train {\clu} to classify the classes 0 versus 6 (``T-shirt'' versus ``Shirt'').
This classifier achieves an AUC of $0.85$, which is competitive with the logistic regression AUC of $0.86$.

Recall that {\clu} gives us a learned link $u$, which per the above discussion also defines an implicit canonical loss.
By training a classifier to distinguish classes 2 versus 4 (``Pullover'' versus ``Coat'') using this loss function,
we achieve an AUC of {$\mathbf{0.879}$}.
This slightly outperforms the {$\mathbf{0.877}$} AUC of logistic regression,
and is also competitive with the $\mathbf{0.879}$ AUC attained by training {\clu} directly on classes 2 versus 4.
This indicates that the loss learned by {\clu} on one domain could be useful in related domains to \textit{another} classification algorithms just training a classifier. To properly develop this possibility is out of the scope of this paper, and as far as we know such a perspective is new in machine learning.

\section{Conclusion}\label{sec-con}
Fitting a loss that complies with Bayes decision theory implies not just to be able to learn a classifier, but also a canonical link of a proper loss, and therefore a proper canonical loss. In a 2011 seminal work, Kakade \textit{et al.} made with the {\sc SLIsotron} algorithm the first attempt at solving this bigger picture of supervised learning. %, even when the approach was limited because of several assumptions and limitations. 
We propose in this paper a more general approach grounded on a general Bregman formulation of differentiable proper canonical losses. 

Experiments tend to confirm the ability of our approach, the \clu, to beat the {\sc SLIsotron}, and compete with classical supervised approaches even when they are informed with the optimal choice of link.
Interestingly, they seem to illustrate the importance of a stability requirement made by our theory. More interesting is perhaps the observation that the loss learned by the \clu~on one domain can be useful to other learning algorithms to fit classifiers on related domains, a \textit{transferability} property of the loss learned that  deserves further thought.

\section*{Acknowledgments}
\label{sec:ackno}

Many thanks to Manfred Warmuth for discussions around the introduction
of Bregman divergences in machine learning.

\bibliographystyle{icml2020}
\bibliography{references,bibgen}

\clearpage
\newpage
%\onecolumn

\section*{Appendix}

\section{Factsheet on Bregman divergences}\label{sec_FactSheetBreg}

We summarize in this section the results we use (both in the main file and in this \supplement) related to
Bregman divergence with convex generator $F$,
\begin{eqnarray}
D_{F}(
  z\|z') & \defeq & F(z) - F(z') - 
  (z-z')F' (z'), \label{eqBreg2}
\end{eqnarray}
where we assume for the sake of simplicity that $F$ is twice
differentiable.\\

\noindent $\triangleright$ \textbf{General properties -- } $D_F$ is
always non-negative, convex in its left parameter, but not always in its right
parameter. Only the divergences corresponding to $F(z) \propto z^2$
are symmetric \citep{bnnBV}.\\

\noindent $\triangleright$ \textbf{$D_F$ is locally proportional to the square loss --
} assuming second order differentiability, we have \citep{nlkMB}:
\begin{eqnarray}
\forall z, z', \exists c \in [z\wedge z', z\vee z']: D_{F}(
  z\|z') & = & \frac{F''(c)}{2} \cdot (z-z')^2. \label{propLPS}
\end{eqnarray}

\noindent $\triangleright$ \textbf{Bregman triangle equality --
} also called the three points property \citep{nlkMB,nmoAS},
\begin{eqnarray}
\forall z, z', z'', D_F(z\|z'') & = & D_F(z\|z') + D_F(z'\|z'') +
                                      (F'(z'') - F'(z'))(z' - z).\label{propBTE}
\end{eqnarray}

\noindent $\triangleright$ \textbf{Invariance to affine terms --
} for any affine function $G(z)$ \citep{bnnBV},
\begin{eqnarray}
\forall z, z', D_{F+G}(z\|z') & = & D_F(z\|z'). \label{propIAT}
  \end{eqnarray}

\noindent $\triangleright$ \textbf{Dual symmetry --
} letting $F^\star$ denote the convex conjugate of $F$, we have \citep{nmoAS},
\begin{eqnarray}
\forall z, z', D_F(z\|z') & = & D_{F^\star}(F'(z')\|F'(z)). \label{propBDS}
  \end{eqnarray}

\noindent $\triangleright$ \textbf{The right population minimizer is
  the mean --
} we have \citep{bmdgCW},
\begin{eqnarray}
\arg\min_z \expect_{\Z}[D_F(\Z\|z)] & = & \expect_{\Z}[\Z] \defeq \mu(\Z). \label{expectMIN}
  \end{eqnarray}

\noindent $\triangleright$ \textbf{Bregman information --
} the Bregman information of random variable $\Z$, defined as $I_F(\Z) \defeq \min_z
\expect_{\Z}[D_F(\Z\|z)]$, satisfies \citep{bmdgCW}
\begin{eqnarray}
I_F(\Z) & = & \expect_{\Z}[D_F(\Z\|\mu(\Z))]. \label{bregINF}
  \end{eqnarray}

\section{Proof of Theorem \ref{thLBREG}}\label{proof_thLBREG}

\noindent  ($\Rightarrow$) The proof assumes basic knowledge about proper losses as in
\citet{rwCB} (and references therein) for example. It comes from \citet[Theorem 1, Corollary 3]{rwCB} and
\citet{samAP} that a differentiable function defines a proper loss iff
there exists a Riemann integrable (eventually
improper in the integrability sense) function $w: (0,1) \rightarrow \mathbb{R}_+$ such that:
\begin{eqnarray}
w(c) = \frac{\properloss'_{-1}(c)}{c} = -
\frac{\properloss'_{1}(c)}{1-c} \:\:, \forall c \in (0,1).\label{propPROPERWEIGHT}
\end{eqnarray}
To simplify notations, we slightly abuse notations and let $\cbr'' \defeq -w$ and define $\cbr'(u) \defeq
\int_a^u \cbr''(z)\mathrm{d}z$ for some adequately chosen constant $a$
(for example, $a = 1/2$ for symmetric proper canonical losses
\cite{nnBD,nnOT}). We denote such a representation of loss functions
their integral representation
\citep[eq. (5)]{rwCB}, as it gives:
\begin{eqnarray}
\partialloss{1}(c) & = &  \int_c^1 - (1-u) \cbr''(u) \mathrm{d}u,
\end{eqnarray}
from which we derive by integrating by parts,
\begin{eqnarray}
\partialloss{1}(c) & = &  -\left[ (1-u) \cbr'(u) \right]_{c}^1 -
\int_c^1 \cbr'(u) \mathrm{d}u \nonumber\\
                   & = & (1-c) \cbr'(c) - \cbr(1) + \cbr(c) \\
                   & = & (-\cbr)(1) - (-\cbr)(c) - (1-c) (-\cbr)'(c) \\
  & = & D_{-\cbr}(1\|c),
\end{eqnarray}
Where $D_{-\cbr}$ is the Bregman divergence with generator $-\cbr$ (we
remind that the conditional Bayes risk of a proper loss is concave
\cite[Section 3.2]{rwCB}). We
get similarly for the partial loss $\partialloss{-1}$ \citep[eq. (5)]{rwCB}:
\begin{eqnarray}
\partialloss{-1}(c) & = & - \int_0^c u \cbr''(u) \mathrm{d}u 
\nonumber\\
& = & - \left[ u\cbr'(u) \right]_{0}^c +
\int_0^c \cbr'(u) \mathrm{d}u \nonumber\\
                    & = & -c \cbr'(c) + \cbr(c) - \cbr(0) \\
                    & = & (- \cbr)(0) - (-\cbr)(c)   - (0 - c) (-\cbr)'(c)\\
  & = & D_{-\cbr}(0\|c).
\end{eqnarray}
We now replace $c$ by the inverse of the link chosen, $\psi$, and we get for
any proper composite loss:
\begin{eqnarray}
\ell(y^*,z) & \defeq & \iver{y^*=1}\cdot \partialloss{1}(\link^{-1}(z)) +
                     \iver{y^*=-1}\cdot \partialloss{-1}(\link^{-1}(z)) \nonumber\\
  & = & D_{-\cbr}(y\|\link^{-1}(z)) ,
\end{eqnarray}
as claimed for the implication $\Rightarrow$. The identity
\begin{eqnarray}
D_{-\cbr}(y\|\link^{-1}(z)) & = & D_{(-\cbr)^\star}(\canolink \circ \link^{-1}(z)\| \canolink(y))
\end{eqnarray}
 follows from the dual symmetry property of Bregman divergences \citep{bnnBV,nmoAS}.\\

\noindent ($\Leftarrow$) Let $\ell(y^*,z) \defeq
D_{-F}(y\|g^{-1}(z))$, some Bregman divergence, where $g: [0,1]
\rightarrow \mathbb{R}$ is invertible. Let $\ell_p (y^*,c) : \mathcal{Y} \times [0,1]
\rightarrow \overline{\mathbb{R}}$ defined by $\ell_p (y^*,c)\defeq \ell(y^*,g(c))$.
We know that the right
population minimizer of any Bregman divergence is the expectation \citep{bmdgCW,nmoAS}, so
$\pi \in \arg\inf_{u} \E_{\Y\sim \pi} \ell_p (\Y, u), \forall \pi \in
[0,1]$ and $\ell_p$ is proper. Therefore $\ell$ is proper composite since $g$ is invertible. 
The conditional Bayes risk of $\ell_p$ is therefore by
definition:
\begin{eqnarray}
  \cbr(\pi) & \defeq & \E_{\Y\sim \pi} \ell_p(\Y, \pi)\\
  & = & F(\pi) + G(\pi) \label{FZU}
  \end{eqnarray}
where $G(\pi) \defeq -  \pi F(1) -  (1-\pi) F(0)$ is affine. Since a
Bregman divergence is invariant by addition of an affine term to its
generator \eqref{propIAT}, we get
\begin{eqnarray}
 \ell_p(y^*,c) &=& 
D_{-F}(y\|c) \\
&= &D_{-\cbr}(y\|c).
  \end{eqnarray}
We now check that if $g = -F'$ then $\ell$ is proper canonical. It
comes from \eqref{FZU}
$(-F')^{-1}(z) = (\canolink)^{-1}(z + K)$ where $K \defeq -(F(1) - F(0))$
is a constant, which is still the inverse of the canonical link since
it is defined up to multiplication or addition by a scalar
\citep{bssLF}. Hence, if $g = -F'$ then $\ell(y^*,z)$ is proper
canonical. Otherwise as previously argued it is proper composite with link $g$ in the more
general case. This completes the proof for the implication $\Leftarrow$, and ends the
proof of Theorem \ref{thLBREG}.

\noindent \textbf{Remark:} symmetric proper
canonical losses (such as the logistic, square or Matsushita losses) admit $\cbr(0) = \cbr(1)$
\cite{nnBD,nnOT}. Hence \eqref{FZU} enforces $\forall \pi \in [0,1]$
\begin{eqnarray}
\pi(F(0) - F(1)) = \cbr(0) = \cbr(1) = (1-\pi)(F(1) - F(0)),
\end{eqnarray}
resulting in $F(1) = F(0)$ and therefore enforcing the constraint $K =
0$ above.

\section{Proof of Theorem \ref{thCLU}}\label{proof_thCLU}

\subsection{Helper results about \clu~and \fit}

To prove the Theorem, we first show several simple helper results.
The first is a simple consequence of the design of $u_t$. We
prove it for the sake of completeness.
\begin{lemma}\label{lemlip}
Let $u_t$ be the function output by \fit~in \clu. Let $z_m \defeq
u_t^{-1}(0)$ and $z_m \defeq u_t^{-1}(1)$. Let $U_t$ be defined as in
\eqref{defU} (main body, with $u \leftarrow u_t$). The following holds true on $u_t$
\begin{eqnarray}
n_{t-1}\cdot (z-z') \leq u_{t}(z) - u_{t}(z') \leq N_{t-1}\cdot(z-z')\:\:,\label{lin1}\\
\frac{1}{N_{t-1}}\cdot (p-p')  \leq u^{-1}_{t}(p) - u^{-1}_{t}(p') \leq \frac{1}{n_{t-1}}\cdot (p-p') \:\:,\label{lin3}
\end{eqnarray}
$\forall z_m \leq z' \leq
z\leq z_M$, $\forall 0\leq p' \leq
p\leq 1$, and the following holds true on $U_t$:
\begin{eqnarray}
\frac{(p-p')^2}{2 N_{t-1}} \leq D_{U^\star_t}(p\| p') \leq \frac{(p-p')^2}{2 n_{t-1}}.\label{propBREG1}
\end{eqnarray}
\end{lemma}
\begin{proof}
We show the right-hand side of ineq. (\ref{lin1}). The left hand side of (\ref{lin1}) follows by symmetry
and ineq (\ref{lin3}) follow after a variable change
from ineq (\ref{lin1}). The proof is a rewriting of the mean-value
Theorem for subdifferentials: consider for example the
case $u_{t}(b) - u_{t}(a) = N'(b-a)$ with $N' > N_{t-1}$ for some $z_m <
a<b < z_M$. Let
\begin{eqnarray}
v(z) & \defeq & 
u_{t}(z) - u_{t}(b)  + N' (b - z)\:\:,
\end{eqnarray} 
and since $v(a) = v(b) = 0$, let $z_* \defeq
\arg\min_z v(z)$, assuming wlog that the min exists. Then $v(z) \geq
v(z_*)$, and equivalently $u_{t}(z) - u_{t}(b)  + N' (b - z) \geq
u_{t}(z_*) - u_{t}(b)  + N' (b - z_*)$ ($\forall z \in [a,b]$), which, after reorganising,
gives $u_{t}(z)\geq u_{t}(z_*) + N'(z - z_*)$, implying $N' \in \partial u_{t}(z_*)$. Pick now
$a\leq z'_*<z_*<z''_* \leq b$ that are linked to $z_*$ by a line segment in
$u_{t}$. At least one of the two segments has slope
$\geq N'$, which is impossible since $N' > N_{t-1}$ and yields a contradiction. The case $a = z_m$ xor
$b = z_M$ reduces to a single segment with slope $\geq N'$, also
impossible.\\

\noindent We now show \eqref{propBREG1}. Let
\begin{eqnarray}
V(p) & \defeq & U^\star_t(b) - U^\star_t(p) - (b-p)u^{-1}_t(z) -A(b-p)^2,
\end{eqnarray}
(remind that $(U^\star_t)' = u^{-1}_t$) where $A$ is chosen so that
$V(a) = 0$, which implies\footnote{This is a simple application of
  Rolle's Theorem to subdifferentials.} since $V(b)=0$ that $\exists c \in
(a,b), 0 \in \partial V(c)$. We have $\partial V(c) \ni -(b-c) c'
- 2A(c-b)$ for any $c' \in \partial u^{-1}_t (c)$, implying $A = c'/2$ for
some $c' \in \partial u^{-1}_t (c)$. Solving for $V(a) = 0$ yields
$D_{U^\star_t}(b\|a) = (c'/2)(b-a)^2$ for some $c'\in \partial u^{-1}_t$ and
since $\mathrm{Im} \partial u^{-1}_t \subset [1/N_{t-1}, 1/n_{t-1}]$
from \eqref{lin3}, we get
\begin{eqnarray}
\frac{(b-a)^2}{2N_{t-1}} \leq D_{U^\star_t}(b\|a) \leq \frac{(b-a)^2}{2n_{t-1}},
\end{eqnarray}
as claimed.
\end{proof}

Note that we indeed have $\hat{y}_1 \defeq u_{t+1}(\ve{w}_{t+1}^\top\ve{x}_1)$ by the design of Step 4 in \clu.
 The second result we need is a direct consequence of Step 3 in \clu.
\begin{lemma}\label{lab}
The following holds for any $t \geq 1, i\in [m]$,
    \begin{eqnarray}
u_{t+1}(\ve{w}_{t+1}^\top \ve{x}_{i}) & \in & u_{t}(\ve{w}_{t+1}^\top
\ve{x}_{i}) \cdot \left[\min\left\{1-\beta_t, \frac{n_t}{N_t}\right\}, \max\left\{1+\alpha_t, \frac{N_t}{n_t}\right\}\right],\label{eqalphabeta}
      \end{eqnarray}
where $\alpha_t, \beta_t \geq 0$ are the stability property parameters
at the current iteration of \clu,
as defined in Definition \ref{defAB} (main file).
\end{lemma}
\begin{proof}
  We prove the upperbound in \eqref{eqalphabeta} by induction. Assuming the property holds for $\ve{x}_{i}$  and considering $\ve{x}_{i+1}$ (recall that indexes are
ordered in increasing value of $\ve{w}_{t+1}^\top \ve{x}_{i}$, see
Step 2 in \clu), we obtain
\begin{eqnarray}
u_{t+1}(\ve{w}_{t+1}^\top \ve{x}_{{i+1}})  & \leq &
u_{t+1}(\ve{w}_{t+1}^\top \ve{x}_{{i}}) + N_t \ve{w}_{t+1}^\top (\ve{x}_{i+1}
- \ve{x}_{i})\label{eqsc11b}\\
& \leq & u_{t+1}(\ve{w}_{t+1}^\top \ve{x}_{{i}}) + \frac{N_t}{n_t} \cdot \left(u_{t}(\ve{w}_{t+1}^\top
\ve{x}_{{i+1}}) - u_{t}(\ve{w}_{t+1}^\top
         \ve{x}_{{i}})\right)\label{eqsc22bb}\\
  & & = \frac{N_t}{n_t} u_{t}(\ve{w}_{t+1}^\top
\ve{x}_{{i+1}}) + u_{t+1}(\ve{w}_{t+1}^\top \ve{x}_{{i}}) - \frac{N_t}{n_t} \cdot u_{t}(\ve{w}_{t+1}^\top
         \ve{x}_{{i}})\:\:.\label{eqsc23bb}
\end{eqnarray}
The first inequality comes from the right interval constraint in
problem (\ref{pbregU}) applied to $u_{t+1}$, ineq. (\ref{eqsc22bb})
comes from Lemma \ref{lemlip} applied to $u_{t}$. We now have two cases.

\noindent \textbf{Case 1} If $N_t/n_t > 1+\alpha_t$,
using the induction hypothesis \eqref{eqalphabeta} yields
$u_{t+1}(\ve{w}_{t+1}^\top \ve{x}_{{i}}) \leq (N_t/n_t) \cdot u_{t}(\ve{w}_{t+1}^\top
\ve{x}_{i})$ and so \eqref{eqsc23bb} becomes
\begin{eqnarray}
u_{t+1}(\ve{w}_{t+1}^\top \ve{x}_{{i+1}}) & \leq &  \frac{N_t}{n_t} u_{t}(\ve{w}_{t+1}^\top
\ve{x}_{{i+1}}) + \frac{N_t}{n_t}  u_{t}(\ve{w}_{t+1}^\top \ve{x}_{{i}}) - \frac{N_t}{n_t} u_{t}(\ve{w}_{t+1}^\top
         \ve{x}_{{i}})\nonumber\\
& & = \frac{N_t}{n_t} u_{t}(\ve{w}_{t+1}^\top
\ve{x}_{{i+1}}).
\end{eqnarray}
\noindent \textbf{Case 2} If $N_t/n_t \leq 1+\alpha_t$, we have this time
from the induction hypothesis $u_{t+1}(\ve{w}_{t+1}^\top \ve{x}_{{i}}) \leq (1+\alpha_t) \cdot u_{t}(\ve{w}_{t+1}^\top
\ve{x}_{i})$, and so we get from \eqref{eqsc23bb},
\begin{eqnarray}
u_{t+1}(\ve{w}_{t+1}^\top \ve{x}_{{i+1}})  & \leq & \frac{N_t}{n_t} u_{t}(\ve{w}_{t+1}^\top
\ve{x}_{{i+1}}) + \left( 1 + \alpha_t - \frac{N_t}{n_t} \right) \cdot u_{t}(\ve{w}_{t+1}^\top
\ve{x}_{{i}}) \nonumber\\
 & \leq & \frac{N_t}{n_t} u_{t}(\ve{w}_{t+1}^\top
\ve{x}_{{i+1}}) + \left( 1 + \alpha_t - \frac{N_t}{n_t} \right) \cdot u_{t}(\ve{w}_{t+1}^\top
\ve{x}_{{i+1}}) \label{eqsc3d3}\\
 & & = (1+\alpha_t) u_{t}(\ve{w}_{t+1}^\top
\ve{x}_{{i+1}}) \label{eqsc44}\:\:,
\end{eqnarray}
where \eqref{eqsc3d3} holds because $\ve{w}_{t+1}^\top
\ve{x}_{{i}}\leq  \ve{w}_{t+1}^\top
\ve{x}_{{i+1}}$ (by assumption) and $u_t$ is non-decreasing.\\

\noindent The proof of the lowerbound in \eqref{eqalphabeta} follows from the
following "symmetric" induction, noting first that the second constraint in problem (\ref{pbregU}) (main file)
  implies the base case, $u_{t+1}(\ve{w}_{t+1}^\top \ve{x}_{1})  \geq (1-\beta_t) u_{t}(\ve{w}_{t+1}^\top
\ve{x}_{1})$, and then, for the general index $i>1$,
\begin{eqnarray}
u_{t+1}(\ve{w}_{t+1}^\top \ve{x}_{{i+1}})  & \geq &
u_{t+1}(\ve{w}_{t+1}^\top \ve{x}_{{i}}) + n_t \ve{w}_{t+1}^\top (\ve{x}_{i+1}
- \ve{x}_{i})\label{eqsc113b}\\
& \geq & u_{t+1}(\ve{w}_{t+1}^\top \ve{x}_{{i}}) + \frac{n_t}{N_t} \cdot \left(u_{t}(\ve{w}_{t+1}^\top
\ve{x}_{{i+1}}) - u_{t}(\ve{w}_{t+1}^\top
\ve{x}_{{i}})\right)\label{eqsc223b}\\
& & = \frac{n_t}{N_t} \cdot u_{t}(\ve{w}_{t+1}^\top
\ve{x}_{{i+1}}) + u_{t+1}(\ve{w}_{t+1}^\top \ve{x}_{{i}}) - \frac{n_t}{N_t} \cdot u_{t}(\ve{w}_{t+1}^\top
\ve{x}_{{i}}) \label{eqsc233bb}.
\end{eqnarray}
The first inequality comes from the left interval constraint in
problem (\ref{pbregU}) applied to $u_{t+1}$, ineq. (\ref{eqsc223b})
comes from Lemma \ref{lemlip} applied to $u_{t}$. Similarly to the
upperbound in \eqref{eqalphabeta}, we now have two cases.

\noindent \textbf{Case 1} If $n_t/N_t \leq 1-\beta_t$, using the induction hypothesis \eqref{eqalphabeta} yields
$u_{t+1}(\ve{w}_{t+1}^\top \ve{x}_{{i}}) \geq (n_t/N_t) \cdot u_{t}(\ve{w}_{t+1}^\top
\ve{x}_{i})$ and so \eqref{eqsc233bb} becomes
\begin{eqnarray}
u_{t+1}(\ve{w}_{t+1}^\top \ve{x}_{{i+1}}) & \geq & \frac{n_t}{N_t} \cdot u_{t}(\ve{w}_{t+1}^\top
\ve{x}_{{i+1}}) + \frac{n_t}{N_t}  \cdot u_{t}(\ve{w}_{t+1}^\top
\ve{x}_{i}) - \frac{n_t}{N_t} \cdot u_{t}(\ve{w}_{t+1}^\top
\ve{x}_{{i}}) \nonumber\\
& & = \frac{n_t}{N_t} \cdot u_{t}(\ve{w}_{t+1}^\top
\ve{x}_{{i+1}}).
\end{eqnarray}
\noindent \textbf{Case 2} If $n_t/N_t > 1-\beta_t$, using the induction hypothesis \eqref{eqalphabeta} yields
$u_{t+1}(\ve{w}_{t+1}^\top \ve{x}_{{i}}) \geq (1-\beta_t) \cdot u_{t}(\ve{w}_{t+1}^\top
\ve{x}_{i})$ and so \eqref{eqsc233bb} becomes
\begin{eqnarray}
u_{t+1}(\ve{w}_{t+1}^\top \ve{x}_{{i+1}}) & \geq & \frac{n_t}{N_t} \cdot u_{t}(\ve{w}_{t+1}^\top
\ve{x}_{{i+1}}) + (1-\beta_t) \cdot u_{t}(\ve{w}_{t+1}^\top
\ve{x}_{i}) - \frac{n_t}{N_t} \cdot u_{t}(\ve{w}_{t+1}^\top
\ve{x}_{{i}}) \nonumber\\
& \geq & \frac{n_t}{N_t} \cdot u_{t}(\ve{w}_{t+1}^\top
\ve{x}_{{i}}) + (1-\beta_t) \cdot u_{t}(\ve{w}_{t+1}^\top
\ve{x}_{i}) - \frac{n_t}{N_t} \cdot u_{t}(\ve{w}_{t+1}^\top
\ve{x}_{{i}}) \label{eqsc3d3bb}\\
& & = (1-\beta_t) \cdot u_{t}(\ve{w}_{t+1}^\top
\ve{x}_{i}).
\end{eqnarray}
\eqref{eqsc3d3bb} holds because $\ve{w}_{t+1}^\top
\ve{x}_{{i}}\leq  \ve{w}_{t+1}^\top
\ve{x}_{{i+1}}$ (by assumption) and $u_t$ is non-decreasing. This
achieves the proof of Lemma \ref{lab}.
\end{proof}
We now analyze the following Bregman loss for $r,t,t'=1, 2, ...$:
\begin{eqnarray}
  \properlossplus{t}{r}(S,\ve{w}_{t'}) & \defeq &
\expect_{{S}}[D_{U_r^\star}(y
\| u_t(\ve{w}_{t'}^\top\ve{x}))] =
\expect_{{S}}\left[D_{U_r}(u_r^{-1}\circ u_t(\ve{w}_{t'}^\top\ve{x}) \| u_r^{-1}(y))\right] \:\:,
\end{eqnarray}
The key to the proof of Theorem \ref{thCLU} is the following Theorem
which breaks down the bound that we have to analyze into several parts.
\begin{theorem}\label{thBTT}
For any $t\geq 1$,
\begin{eqnarray}
\properlossplus{t+1}{t+1}(S,\ve{w}_{t+1}) & \leq & \properlossplus{t}{t}(S,\ve{w}_{t})
                                              - \expect_{{S}}[D_{U_t}(
        \ve{w}_{t}^\top\ve{x}\| u^{-1}_{t} \circ
                                              u_{t+1}(\ve{w}_{t+1}^\top\ve{x}))]
                                              - L_{t+1} - Q_{t+1},\nonumber
\end{eqnarray}
where 
\begin{eqnarray}
 L_{t+1}  & \defeq & \expect_{{S}}[(\ve{w}_{t}^\top\ve{x} -
      u^{-1}_{t}\circ u_{t+1}(\ve{w}_{t+1}^\top\ve{x}))\cdot
      (u_{t+1}(\ve{w}_{t+1}^\top\ve{x}) -
                     u_{t}(\ve{w}_{t+1}^\top\ve{x}))],\nonumber\\
Q_{t+1} & \defeq & \expect_{{S}}[(\ve{w}_{t}^\top\ve{x} -
      u^{-1}_{t}\circ u_{t+1}(\ve{w}_{t+1}^\top\ve{x}))\cdot
      (u_{t}(\ve{w}_{t+1}^\top\ve{x}) - y)] -                              \left(\frac{N_{t-1}}{n_t} -1\right)\cdot
                                                   \properlossplus{t+1}{t}(S,\ve{w}_{t+1}).\nonumber
\end{eqnarray}
\end{theorem}
\begin{proof}
We have the following derivations:
\begin{eqnarray}
\properlossplus{t}{t}(S,\ve{w}_{t}) & = &
                                          \expect_{{S}}[D_{U_t^\star}(y
\| u_t(\ve{w}_{t}^\top\ve{x}))]\nonumber\\
  & = & \expect_{{S}}[D_{U_t^\star}(y
\| u_{t+1}(\ve{w}_{t+1}^\top\ve{x}))] + \expect_{{S}}[D_{U_t^\star}(u_{t+1}(\ve{w}_{t+1}^\top\ve{x})
        \|u_t(\ve{w}_{t}^\top\ve{x}))]\nonumber\\
  & & + \expect_{{S}}[((U_t^\star)'(u_t(\ve{w}_{t}^\top\ve{x})) -
      (U_t^\star)'(u_{t+1}(\ve{w}_{t+1}^\top\ve{x})))\cdot
      (u_{t+1}(\ve{w}_{t+1}^\top\ve{x}) - y)]\label{useBTE1}\\
  & = & \properlossplus{t+1}{t}(S,\ve{w}_{t+1}) + \expect_{{S}}[D_{U_t}(
        \ve{w}_{t}^\top\ve{x}\| u^{-1}_{t} \circ u_{t+1}(\ve{w}_{t+1}^\top\ve{x}))]\nonumber\\
  & & + \underbrace{\expect_{{S}}[(\ve{w}_{t}^\top\ve{x} -
      u^{-1}_{t}\circ u_{t+1}(\ve{w}_{t+1}^\top\ve{x}))\cdot
      (u_{t+1}(\ve{w}_{t+1}^\top\ve{x}) - y)]}_{\defeq \Delta_{t+1}}\label{useBINVBDS1}.
\end{eqnarray}
\eqref{useBTE1} follows from the Bregman triangle equality
\eqref{propBTE}. \eqref{useBINVBDS1} follows from $(U_t^\star)' \defeq
u_t^{-1}$ and \eqref{propBDS}. Reordering, we get:
\begin{eqnarray}
  \properlossplus{t+1}{t}(S,\ve{w}_{t+1}) & = &
                                              \properlossplus{t}{t}(S,\ve{w}_{t})
                                              - \expect_{{S}}[D_{U_t}(
        \ve{w}_{t}^\top\ve{x}\| u^{-1}_{t} \circ
                                              u_{t+1}(\ve{w}_{t+1}^\top\ve{x}))]
                                              - \Delta_{t+1},\label{useBTE2}
\end{eqnarray}
and we further split $\Delta_{t+1}$ in two: $\Delta_{t+1} \defeq
F_{t+1} + L_{t+1}$, where
\begin{eqnarray}
  F_{t+1} & \defeq & \expect_{{S}}[(\ve{w}_{t}^\top\ve{x} -
      u^{-1}_{t}\circ u_{t+1}(\ve{w}_{t+1}^\top\ve{x}))\cdot
      (u_{t}(\ve{w}_{t+1}^\top\ve{x}) - y)]\label{defFT}, \\
  L_{t+1}  & \defeq & \expect_{{S}}[(\ve{w}_{t}^\top\ve{x} -
      u^{-1}_{t}\circ u_{t+1}(\ve{w}_{t+1}^\top\ve{x}))\cdot
      (u_{t+1}(\ve{w}_{t+1}^\top\ve{x}) - u_{t}(\ve{w}_{t+1}^\top\ve{x}))]\label{defLT}.
\end{eqnarray}
We now have the following Lemma.
\begin{lemma}\label{lemapproxab}
The following holds for any $t\geq 0$:
\begin{eqnarray}
\properlossplus{t+1}{t+1}(S,\ve{w}_{t+1}) & \leq & \frac{N_{t-1}}{n_t} \cdot
                                                   \properlossplus{t+1}{t}(S,\ve{w}_{t+1}).
\end{eqnarray}
\end{lemma}
\begin{proof}
We use Lemma
\ref{lemlip} and we get:
\begin{eqnarray}
D_{U^*_{t+1}}(p
\| p') & \leq & \frac{1}{2n_t}\cdot (p - p')^2\nonumber\\
 & \leq & \frac{N_{t-1}}{n_t} \cdot D_{U^*_{t}}(p
\| p')\:\:,
\end{eqnarray}
from which we just compute the expectation in
$\properlossplus{t+1}{.}(S,\ve{w}_{t+1}) $ and get the result as claimed.
\end{proof}
Putting altogether \eqref{useBTE2}, \eqref{defFT}, \eqref{defLT} and
Lemma \ref{lemapproxab} yields, $\forall t\geq 1$,
\begin{eqnarray}
\properlossplus{t+1}{t+1}(S,\ve{w}_{t+1}) & \leq & 
                                                   \properlossplus{t+1}{t}(S,\ve{w}_{t+1})
                                                   +
                                                   \left(\frac{N_{t-1}}{n_t} -1\right)\cdot
                                                   \properlossplus{t+1}{t}(S,\ve{w}_{t+1})\nonumber\\
& = & \properlossplus{t}{t}(S,\ve{w}_{t})
                                              - \expect_{{S}}[D_{U_t}(
        \ve{w}_{t}^\top\ve{x}\| u^{-1}_{t} \circ
                                              u_{t+1}(\ve{w}_{t+1}^\top\ve{x}))]
                                              - L_{t+1} \nonumber\\
      & & - \left(                F_{t+1} -                              \left(\frac{N_{t-1}}{n_t} -1\right)\cdot
                                                   \properlossplus{t+1}{t}(S,\ve{w}_{t+1})\right),
\end{eqnarray}
as claimed. This ends the proof of Theorem \ref{thBTT}.
\end{proof}
Last, we provide a simple result about the gradient step in Step 1.
\begin{lemma}\label{gradUPDT}
Let $\hve{\mu}_{y} \defeq \expect_{{S}}[y\cdot \ve{x}]$ and $\hve{\mu}_{t} \defeq
\expect_{{S}}[\hat{y}_{t} \cdot \ve{x}]$. The gradient update for
\eqref{pbregW} in Step 1 of the \clu~
yields the following update to get $\ve{w}_{t+1}$, for some learning
rate $\upeta > 0$:
\begin{eqnarray}
\ve{w}_{t+1} & \leftarrow & \ve{w}_{t} + \upeta\cdot(\hve{\mu}_{y} -
\hve{\mu}_{t})\:\:. \label{gupd}
\end{eqnarray}
\end{lemma}
\begin{proof}
We trivially have $\nabla_{\ve{w}}
\expect_{{S}}[D_{U_{t}}(\ve{w}^\top\ve{x}\|u^{-1}_{t}(y))] =
\expect_{{S}}[u_t(\ve{w}^\top\ve{x})\cdot \ve{x} - y \cdot
\ve{x}] = \expect_{{S}}[u_t(\ve{w}^\top\ve{x})\cdot \ve{x}]
- \hve{\mu}_{y}$, from which we get, for some $\upeta > 0$ the
gradient update:
\begin{eqnarray}
\ve{w}_{t+1} & \leftarrow & \ve{w}_{t} - \upeta\cdot \nabla_{\ve{w}}
\expect_{{S}}[D_{U_{t}}(\ve{w}^\top\ve{x}\|u^{-1}_{t}(y))]_{|
                           \ve{w} = \ve{w}_t}
= \ve{w}_{t} + \upeta\cdot(\hve{\mu}_{y} -
\hve{\mu}_{t})\:\:, \label{gupl}
\end{eqnarray}
as claimed.
\end{proof}

\subsection{Proof of Theorem \ref{thCLU}}

\begin{lemma}\label{lir1}
    $L_{t+1} \geq 0$, $\forall t$.
\end{lemma}
\begin{proof}
We show that the Lemma is a consequence of the fitting of $u_{t+1}$ by \fit~from Step 3 in \clu.  
The proof elaborates on the proofsketch of Lemma 2 of \cite{kkksEL}.
Denote for short $N_i \defeq N_t \ve{w}_{t+1}^\top \ve{x}_{i}$ and $n_i
\defeq n_t \ve{w}_{t+1}^\top \ve{x}_{i}$.
We introduce two $(m-1)$-dim vectors of Lagrange multipliers $\ve{{\lambdaleft}}$
and $\ve{{\lambdaright}}$ for the top left and right interval
constraints and two multipliers $\rholeft$
and $\rhoright$ for the additional bounds on $\hat{y}_{1} $ and $\hat{y}_{m} $ respectively. This gives the
Lagrangian,
\begin{eqnarray}
{\mathcal{L}}(\hve{y},S|\ve{{\lambdaleft}}, \ve{{\lambdaright}}, \rholeft, \rhoright) & \defeq & \expect_{{\mathcal{S}}}[
D_{U^\star_{t}}(\hat{y}\|y_t)] + \sum_{i=1}^{m-1} {\lambdaleft}_i \cdot (\hat{y}_i
                                                                          - \hat{y}_{i+1} + n_{i+1} - n_{i}) \nonumber\\
  & & + \sum_{i=1}^{m-1}{\lambdaright}_i \cdot
(\hat{y}_{i+1} - \hat{y}_i - N_{i+1} + N_i) + \rholeft \cdot
       - \hat{y}_1  + \rhoright \cdot
(\hat{y}_m-1) \:\:,\nonumber
\end{eqnarray}
where we let $q_i\defeq u_{t}(\ve{w}_{t+1}^\top \ve{x}_{i})$ for
readability and we adopt the convention of \citet[Chapter 5]{bvCO} for constraints. Letting $\ve{\omega} \in \bigtriangleup_m$ (the $m$-dim probability simplex) denote the weight
vector of the examples ins $S$, we get
the following KKT conditions for the optimum:
\begin{eqnarray}
\omega_i(u^{-1}_{t}(\hat{y}_i) - u^{-1}_{t}(\hat{y}_{ti})) + {\lambdaleft}_i -
{\lambdaleft}_{(i-1)} + {\lambdaright}_{(i-1)} - {\lambdaright}_{i} & = & 0\:\:, \forall i = 2, 3,
..., m-1\:\:,\label{eq1}\\
\omega_i (u^{-1}_{t}(\hat{y}_1) - u^{-1}_{t}(\hat{y}_{t1})) + {\lambdaleft}_{1} -
{\lambdaright}_{1} -\rholeft & = & 0\:\:, \label{eq2}\\
\omega_i (u^{-1}_{t}(\hat{y}_m) - u^{-1}_{t}(\hat{y}_{tm})) - {\lambdaleft}_{(m-1)} +
{\lambdaright}_{(m-1)} +
  {\rhoright} & = & 0\:\:, \label{eq3}\\
\hat{y}_{i+1} - \hat{y}_i & \in & [n_{i+1} - n_i, N_{i+1} -
 N_i]\:\:, \forall i \in [m-1]\:\:,\\
\hat{y}_1 & \geq & 0\:\:,\\
\hat{y}_m & \leq & 1\:\:,\\
{\lambdaleft}_i \cdot  (\hat{y}_i
- \hat{y}_{i+1} +n_{i+1} - n_{i})  & = & 0\:\:, \forall i \in [m-1]\:\:,\label{kktlambda}\\
{\lambdaright}_i \cdot (\hat{y}_{i+1} - \hat{y}_i - N_{i+1} + N_i)  & = & 0\:\:,
\forall i \in [m-1]\:\:,\label{kktrho}\\
\rholeft \cdot
       - \hat{y}_1 & = & 0\:\:, \label{kktpil}\\
\rhoright \cdot  (1 - \hat{y}_m)  & = & 0\:\:, \label{kktpir}\\
\ve{{\lambdaleft}}, \ve{{\lambdaright}} & \succeq &
                                                    \ve{0}\:\:,\nonumber\\
\rholeft, \rhoright & \geq & 0 \nonumber\:\:.
\end{eqnarray}
For $i = 1, 2, ..., m$, we define 
\begin{eqnarray}
\sigma_i & \defeq & \sum_{j=i}^m \omega_j  (u^{-1}_{t}(\hat{y}_{tj}) -
u^{-1}_{t}(\hat{y}_j))\nonumber.
\end{eqnarray} 
We note that by summing the corresponding subset of (\ref{eq1} ---
\ref{eq3}), we get 
\begin{eqnarray}
\sigma_i & = & {\lambdaright}_{(i-1)} - {\lambdaleft}_{(i-1)} +\rhoright \:\:, \forall i \in \{2, 3,
..., m\}\:\:,\\
\sigma_1 & = & - \rholeft + \rhoright\:\:.\label{defsig2}
\end{eqnarray}
Letting $\hat{y}_0$ and $q_0$ denote
any identical reals, we obtain:
\begin{eqnarray}
\sum_{i=1}^{m} \omega_i  (u^{-1}_{t}(\hat{y}_{ti}) -
u^{-1}_{t}(\hat{y}_i)) \cdot (\hat{y}_i - q_i)
& = &
\sum_{i=1}^{m} \sigma_i \cdot ((\hat{y}_i - q_i)-(\hat{y}_{(i-1)} - q_{(i-1)})) \:\:,\label{prodlem}
\end{eqnarray}
which we are going to show is non-negative, which is the statement of
the Lemma, in two steps:\\
\noindent \textbf{Step 1} -- We show, for any $i\geq 1$,
\begin{eqnarray}
(\sigma_i - \rhoright)\cdot ((\hat{y}_i - q_i)-(\hat{y}_{(i-1)} -
  q_{(i-1)})) & \geq & 0.\label{bfinRHO1}
\end{eqnarray}
We have four cases:\\
\noindent \textbf{Case 1.1}  $i>1$, $\sigma_i - \rhoright> 0$. In this
case,
${\lambdaright}_{(i-1)} > {\lambdaleft}_{(i-1)}$, implying ${\lambdaright}_{(i-1)} > 0$ and so from
eq. (\ref{kktrho}), $\hat{y}_{i} - \hat{y}_{(i-1)} - N_{i} + N_{(i-1)} =
0$, and so $\hat{y}_{i} - \hat{y}_{(i-1)} = N_t \ve{w}_{t+1}^\top
(\ve{x}_{i}-\ve{x}_{(i-1)})$. Lemma \ref{lemlip} applied to $u_t$ gives
\begin{eqnarray}
u_{t}(\ve{w}_{t+1}^\top \ve{x}_{i}) - u_{t}(\ve{w}_{t+1}^\top
\ve{x}_{(i-1)}) & \leq & N_t \ve{w}_{t+1}^\top
(\ve{x}_{i}-\ve{x}_{(i-1)})\:\:,
\end{eqnarray}
and so $\hat{y}_{i} - \hat{y}_{(i-1)} \geq q_i - q_{(i-1)}$, that is, $(\hat{y}_{i} - q_i) -
(\hat{y}_{(i-1)}  - q_{(i-1)}) \geq 0$.\\
\noindent \textbf{Case 1.2}  $i>1$, $\sigma_i - \rhoright< 0$. In this case, 
${\lambdaleft}_{(i-1)} > {\lambdaright}_{(i-1)}$, implying ${\lambdaleft}_{(i-1)} > 0$, and so from
eq. (\ref{kktlambda}), $\hat{y}_{(i-1)}
- \hat{y}_{i} + n_{i} - n_{(i-1)} = 0$ and so $\hat{y}_{i} - \hat{y}_{(i-1)} = n_t \ve{w}_{t+1}^\top
(\ve{x}_{i}-\ve{x}_{(i-1)})$. Lemma \ref{lemlip} applied to $u_t$ also gives
\begin{eqnarray}
u_{t}(\ve{w}_{t+1}^\top \ve{x}_{i}) - u_{t}(\ve{w}_{t+1}^\top
\ve{x}_{(i-1)}) & \geq & n_t \ve{w}_{t+1}^\top
(\ve{x}_{i}-\ve{x}_{(i-1)})\:\:,
\end{eqnarray}
and so $q_i - q_{(i-1)} \geq \hat{y}_{i} - \hat{y}_{(i-1)}$, or,
equivalently, $(\hat{y}_{i} -  q_{i})-(\hat{y}_{(i-1)} - q_{(i-1)}) \leq
0$.\\
\noindent \textbf{Case 1.3}  $i=1$, $\rholeft> 0$. The case $i=1$
yields $\sigma_1 - \rhoright = -\rholeft$. It comes from KKT
condition \eqref{kktpil} that $\hat{y}_1 = 0$, and since $q_1 \geq 0$
(because of \fit),
we get $\sigma_1 - \rhoright < 0, \hat{y}_1 -
q_1 \leq 0$ and since $\hat{y}_{0} = q_{0}$, we get the statement of \eqref{bfinRHO1}.\\
\noindent \textbf{Case 1.4}  $i=1$, $\rholeft= 0$. We obtain $\sigma_1 -
\rhoright = 0$ and so \eqref{bfinRHO1} immediately holds.\\

\noindent \textbf{Step 2} -- We sum \eqref{bfinRHO1} for $i\in [m]$, getting
\begin{eqnarray}
\sum_{i=1}^{m} \sigma_i \cdot ((\hat{y}_i - q_i)-(\hat{y}_{(i-1)} -
  q_{(i-1)})) & \geq & \sum_{i=1}^{m} \rhoright\cdot ((\hat{y}_i - q_i)-(\hat{y}_{(i-1)} -
  q_{(i-1)}))\nonumber\\
& & = \rhoright\cdot (\hat{y}_m - q_m).\label{bfinRHO3}
\end{eqnarray}
We show that the right-hand side of \eqref{bfinRHO3} is
non-negative. Indeed, it is immediate if $\rhoright= 0$, and if
$\rhoright> 0$, then it comes from KKT
condition \eqref{kktpir} that $\hat{y}_1 = 1$, and since $q_m \leq 1$
(because of \fit), we get $\rhoright\cdot (\hat{y}_m -
q_m) = \rhoright\cdot (1 -
q_m) \geq 0$.\\

\noindent  To summarize our two steps, we have shown that
\begin{eqnarray}
\sum_{i=1}^{m} \sigma_i \cdot ((\hat{y}_i - q_i)-(\hat{y}_{(i-1)} -
  q_{(i-1)})) & \geq & \rhoright\cdot (\hat{y}_m - q_m) \geq 0,\nonumber,
\end{eqnarray}
which brings from \eqref{prodlem} that 
\begin{eqnarray}
\expect_{{\mathcal{S}}}[(u^{-1}_{t}(\hat{y}_{t}) -
u^{-1}_{t}(\hat{y}_{t+1})) \cdot (\hat{y}_{t+1} - u_{t}(\ve{w}_{t+1}^\top \ve{x}))] & \geq & 0\:\:,
\end{eqnarray}
which after using the fact that \fit~guarantees $\hat{y}_{t+1} =
u_{t+1}(\ve{w}_{t+1}^\top\ve{x}), \forall t$, yields
\begin{eqnarray}
\expect_{{\mathcal{S}}}[(\ve{w}_{t}^\top\ve{x} -
u^{-1}_{t} \circ u_{t+1}(\ve{w}_{t+1}^\top\ve{x})) \cdot (u_{t+1}(\ve{w}_{t+1}^\top\ve{x}) - u_{t}(\ve{w}_{t+1}^\top \ve{x}))] & \geq & 0\:\:,
\end{eqnarray}
which is the statement of Lemma \ref{lir1}.
\end{proof}
We recall
\begin{eqnarray}
\hve{\mu}_{y} & \defeq & \expect_{{\mathcal{S}}}[y\cdot \ve{x}],\\
\hve{\mu}_{t} & \defeq &
\expect_{{\mathcal{S}}}[\hat{y}_{t} \cdot \ve{x}], \forall t\geq 1,
\end{eqnarray}
Finally, we let
\begin{eqnarray}
p^*_t & \defeq & \max\{\expect_{{S}}[y],
                 \expect_{{S}}[u_t(\ve{w}_{t+1}^\top\ve{x})]\} \:\:
                 (\in [0,1]).
\end{eqnarray}
\begin{lemma}\label{lemf}
Fix any lowerbound $\delta_t > 0$ such that
\begin{eqnarray}
\frac{\|\hve{\mu}_{y} -
\hve{\mu}_{t}\|_2 }{X} & \geq & 2\sqrt{p^*_t \delta_t}.\label{condMUX}
\end{eqnarray}
Fix any $\gamma_t$ satisfying:
\begin{eqnarray}
\gamma_t & \in & \left[ 0, \sqrt{\frac{\delta_t}{2(2+\delta_t)}} \right],
\end{eqnarray}
and learning rate
\begin{eqnarray}
\upeta & = & \frac{1-\gamma_t}{2 N_t X^2} \cdot \left( 1-
      \frac{\delta_t(2+\delta_t)}{(1+\delta_t)^2} \cdot \frac{p^*_t
X}{\|\hve{\mu}_{y} -
\hve{\mu}_{t}\|_2} \right).
\end{eqnarray}
Suppose $\alpha_t, \beta_t \leq
\delta_t /(1+\delta_t)$ and 
\begin{eqnarray}
\frac{N_t}{n_t}, \frac{N_{t-1}}{n_t} & \leq & 1 + \frac{\delta_t}{1+\delta_t}.
\end{eqnarray}
Then
\begin{eqnarray}
F_{t+1} & \geq & \left(\frac{N_{t-1}}{n_t}
  -1\right)\cdot\frac{p^*_t}{2n_t}+ \frac{p^*_t \delta_t}{n_t(1+\delta_t)}\label{eqBINFF2}.
\end{eqnarray}
\end{lemma}
\noindent \textbf{Remark}: it can be shown from \eqref{condMUX} (see also \ref{binfETA}) that
$\upeta$ belongs to the following interval:
\begin{eqnarray}
\upeta & \in & \frac{1-\gamma_t}{2 N_t X^2} \cdot \left( 1 - \frac{\sqrt{\delta_t p^*_t} (2+\delta_t)}{2(1+\delta_t)^2} \cdot \left[\sqrt{\delta_t p^*_t}, 1\right]\right).\nonumber
\end{eqnarray}
Also, since $\|\hve{\mu}_{y} -
\hve{\mu}_{t}\|_2\leq 2X$, \eqref{condMUX} implies 
\begin{eqnarray}
\delta_t & \leq & \frac{1}{p^*_t}. 
\end{eqnarray}
\begin{proof}
The following two facts are consequences of Lemmata \ref{gradUPDT}, \ref{lemlip} and the
continuity of $u_t$: $\forall i\in [m]$,
\begin{eqnarray}
\exists p_i \in [N^{-1}, n^{-1}] : u_t^{-1}\circ u_{t+1}(\ve{w}_{t+1}^\top\ve{x}_i) & = & u_t^{-1}\circ
                                                                                          u_{t}(\ve{w}_{t+1}^\top\ve{x}_i) \nonumber\\
  & & +
p_i \cdot
(u_{t+1}(\ve{w}_{t+1}^\top\ve{x}_i)-u_{t}(\ve{w}_{t+1}^\top\ve{x}_i))
      \nonumber\\
  & = & \ve{w}_{t+1}^\top\ve{x}_i+
p_i \cdot
(u_{t+1}(\ve{w}_{t+1}^\top\ve{x}_i)-u_{t}(\ve{w}_{t+1}^\top\ve{x}_i)) \:\:,\label{f11}\\
\exists r_i \in [n, N]: u_{t}
  (\ve{w}_{t+1}^\top\ve{x}_i) & = & u_{t}
  (\ve{w}_{t}^\top\ve{x}_i)  + r_i \cdot (\ve{w}_{t+1}-\ve{w}_{t})^\top \ve{x}_i\nonumber\\
  & = & u_{t}
  (\ve{w}_{t}^\top\ve{x}_i)  + \upeta r_i \cdot (\hve{\mu}_{y} -
\hve{\mu}_{t})^\top \ve{x}_i \:\:.\label{f12}
\end{eqnarray}
Folding \eqref{f11} and \eqref{f12} in $F_{t+1}$, we get:
\begin{eqnarray}
  F_{t+1} & = &    \expect_{{\mathcal{S}}}\left[ (u^{-1}_{t} \circ u_{t+1}
  (\ve{w}_{t+1}^\top\ve{x})  - \ve{w}_{t}^\top\ve{x})\cdot (y
  - u_{t}
  (\ve{w}_{t+1}^\top\ve{x}))\right]\nonumber\\
        & = & \expect_{{\mathcal{S}}}\left[ 
              \left\{
              \begin{array}{l}
                \left( \ve{w}_{t+1}^\top\ve{x} - \ve{w}_{t}^\top\ve{x}+
p\cdot (u_{t+1}(\ve{w}_{t+1}^\top\ve{x})-u_{t}(\ve{w}_{t+1}^\top\ve{x}))
                \right) \\
                \cdot \left(y - u_{t}
  (\ve{w}_{t}^\top\ve{x})  - \upeta r \cdot (\hve{\mu}_{y} -
                \hve{\mu}_{t})^\top \ve{x}\right)
              \end{array}\right.\right]\nonumber\\
        & = & \expect_{{\mathcal{S}}}\left[ 
              \left\{
              \begin{array}{l}
                \left( \upeta \cdot (\hve{\mu}_{y} -
\hve{\mu}_{t})^\top \ve{x} +
p\cdot (u_{t+1}(\ve{w}_{t+1}^\top\ve{x})-u_{t}(\ve{w}_{t+1}^\top\ve{x}))
                \right) \\
                \cdot \left(y - u_{t}
  (\ve{w}_{t}^\top\ve{x})  - \upeta r \cdot (\hve{\mu}_{y} -
                \hve{\mu}_{t})^\top \ve{x}\right)
              \end{array}\right.\right]\label{eqGRAD}\\
   & = & \upeta \cdot (\hve{\mu}_{y} -
\hve{\mu}_{t})^\top \expect_{{\mathcal{S}}}\left[y \cdot\ve{x} - u_{t}
         (\ve{w}_{t}^\top\ve{x}) \cdot\ve{x}\right] \nonumber\\
   & & +
         \expect_{{\mathcal{S}}}\left[p \cdot
         (u_{t+1}(\ve{w}_{t+1}^\top\ve{x})-u_{t}(\ve{w}_{t+1}^\top\ve{x}))
         \cdot (y - u_{t}
  (\ve{w}_{t}^\top\ve{x}))
       \right]\nonumber\\
  & & - \upeta^2\cdot \expect_{{\mathcal{S}}}\left[r \cdot ((\hve{\mu}_{y} -
\hve{\mu}_{t})^\top \ve{x})^2
      \right]\nonumber\\
  & & - \upeta\cdot (\hve{\mu}_{y} -
\hve{\mu}_{t})^\top \expect_{{\mathcal{S}}}\left[p r
(u_{t+1}(\ve{w}_{t+1}^\top\ve{x})-u_{t}(\ve{w}_{t+1}^\top\ve{x}))
\cdot \ve{x}\right]\nonumber\\
  & = & \upeta \cdot \|\hve{\mu}_{y} -
\hve{\mu}_{t}\|_2^2 +
         \underbrace{\expect_{{\mathcal{S}}}\left[p \cdot
         (u_{t+1}(\ve{w}_{t+1}^\top\ve{x})-u_{t}(\ve{w}_{t+1}^\top\ve{x}))
         \cdot (y - u_{t}
  (\ve{w}_{t}^\top\ve{x}))
       \right]}_{\defeq A}\nonumber\\
  & & - \underbrace{\upeta^2\cdot \expect_{{\mathcal{S}}}\left[r \cdot ((\hve{\mu}_{y} -
\hve{\mu}_{t})^\top \ve{x})^2
      \right]}_{\defeq B}\nonumber\\
  & & - \underbrace{\upeta\cdot (\hve{\mu}_{y} -
\hve{\mu}_{t})^\top \expect_{{\mathcal{S}}}\left[p r
(u_{t+1}(\ve{w}_{t+1}^\top\ve{x})-u_{t}(\ve{w}_{t+1}^\top\ve{x}))
\cdot \ve{x}\right]}_{\defeq C}\:\:.\label{feq1p}
\end{eqnarray}
We now bound lowerbound $A$ and upperbound $B, C$. 
Lemma \ref{lab}
brings
\begin{eqnarray}
\min\left\{-\beta_t, \frac{n_t}{N_t}-1\right\} \cdot u_{t}(\ve{w}_{t+1}^\top\ve{x}) \leq u_{t+1}(\ve{w}_{t+1}^\top\ve{x})-u_{t}(\ve{w}_{t+1}^\top\ve{x}),\label{genBOUND1}
\end{eqnarray}
and
\begin{eqnarray}
u_{t+1}(\ve{w}_{t+1}^\top\ve{x})-u_{t}(\ve{w}_{t+1}^\top\ve{x}) \leq
  \max\left\{\alpha_t,\frac{N_t}{n_t} -1 \right\} \cdot
  u_{t}(\ve{w}_{t+1}^\top\ve{x}) ,\label{genBOUND2}
\end{eqnarray}
and so we get
\begin{eqnarray}
A & \defeq & \expect_{{\mathcal{S}}}\left[p \cdot
         (u_{t+1}(\ve{w}_{t+1}^\top\ve{x})-u_{t}(\ve{w}_{t+1}^\top\ve{x}))
         \cdot (y - u_{t}
  (\ve{w}_{t}^\top\ve{x}))
       \right] \nonumber\\
& = & \expect_{{\mathcal{S}}}\left[p y \cdot
         (u_{t+1}(\ve{w}_{t+1}^\top\ve{x})-u_{t}(\ve{w}_{t+1}^\top\ve{x}))\right]
      -  \expect_{{\mathcal{S}}}\left[p u_{t}
  (\ve{w}_{t}^\top\ve{x}) \cdot
         (u_{t+1}(\ve{w}_{t+1}^\top\ve{x})-u_{t}(\ve{w}_{t+1}^\top\ve{x}))\right]\nonumber\\
& \geq & - \frac{1}{n_t} \cdot \max\left\{\beta_t, 1-\frac{n_t}{N_t}\right\} \expect_{{\mathcal{S}}}\left[u_{t}(\ve{w}_{t+1}^\top\ve{x})\right] - \frac{1}{n_t} \cdot \max\left\{\alpha_t,\frac{N_t}{n_t} -1 \right\}  \expect_{{\mathcal{S}}}\left[u_{t}(\ve{w}_{t+1}^\top\ve{x}) u_{t}(\ve{w}_{t}^\top\ve{x})\right] \nonumber\\
& \geq & - \frac{1}{n_t} \cdot \max\left\{\alpha_t,\beta_t, 1-\frac{n_t}{N_t}, \frac{N_t}{n_t} -1 \right\} \expect_{{\mathcal{S}}}\left[u_{t}(\ve{w}_{t+1}^\top\ve{x})\right]\nonumber\\
& \geq & - \frac{1}{n_t} \cdot \max\left\{\alpha_t,\beta_t, \frac{N_t}{n_t} -1 \right\} p^*_t  ,
\end{eqnarray}
since $y \in \{0,1\}$, $U_t \leq 1$ and $1-(1/z) \leq z-1$ for $z\geq 0$.
Cauchy-Schwartz inequality and \eqref{f12} yield
\begin{eqnarray}
B & \leq & \upeta^2\cdot \expect_{{\mathcal{S}}}\left[r \cdot \|\hve{\mu}_{y} -
\hve{\mu}_{t}\|_2^2 \|\ve{x}\|_2^2
           \right]\nonumber\\
  & \leq & \upeta^2 N X^2 \cdot \|\hve{\mu}_{y} -
\hve{\mu}_{t}\|_2^2 .
  \end{eqnarray}
We also have successively because of Cauchy-Schwartz inequality, the
triangle inequality, Lemma \ref{lemlip} and \eqref{genBOUND2}
\begin{eqnarray}
C & \leq & \upeta\cdot\|\hve{\mu}_{y} -
\hve{\mu}_{t}\|_2 \cdot\|\expect_{{\mathcal{S}}}\left[ p r
(u_{t+1}(\ve{w}_{t+1}^\top\ve{x})-u_{t}(\ve{w}_{t+1}^\top\ve{x}))
\cdot \ve{x}\right] \|_2 \nonumber\\
& \leq & \upeta\cdot\|\hve{\mu}_{y} -
\hve{\mu}_{t}\|_2 \cdot\expect_{{\mathcal{S}}} [ p r
|u_{t+1}(\ve{w}_{t+1}^\top\ve{x})-u_{t}(\ve{w}_{t+1}^\top\ve{x})|
\cdot \|\ve{x}\|_2] \nonumber\\
 & \leq & \frac{\upeta N_t}{n_t}\cdot\|\hve{\mu}_{y} -
          \hve{\mu}_{t}\|_2 \cdot\expect_{{\mathcal{S}}} [ 
|u_{t+1}(\ve{w}_{t+1}^\top\ve{x})-u_{t}(\ve{w}_{t+1}^\top\ve{x})|
\cdot \|\ve{x}\|_2] \nonumber\\
& \leq & \frac{\upeta N_t \max\left\{\alpha_t,\frac{N_t}{n_t} -1
         \right\}  X}{n_t}\cdot\|\hve{\mu}_{y} -
          \hve{\mu}_{t}\|_2 \cdot \expect_{{\mathcal{S}}}\left[u_{t}(\ve{w}_{t+1}^\top\ve{x})\right]\nonumber\\
& \leq & \frac{\upeta N_t \max\left\{\alpha_t,\frac{N_t}{n_t} -1
         \right\}  X p^*_t }{n_t}\cdot\|\hve{\mu}_{y} -
          \hve{\mu}_{t}\|_2 \:\:.\label{as2p}
\end{eqnarray}
We thus get
\begin{eqnarray}
\lefteqn{F_{t+1} - \left(\frac{N_{t-1}}{n_t}
  -1\right)\cdot\frac{p^*_t }{n_t} }\nonumber\\
& \geq & \upeta \cdot \|\hve{\mu}_{y} -
\hve{\mu}_{t}\|_2^2 -\frac{1}{n_t} \cdot \max\left\{\alpha_t,\beta_t, \frac{N_t}{n_t} -1 \right\} p^*_t  -
         \left(\frac{N_{t-1}}{n_t}
         -1\right)\cdot\frac{p^*_t }{n_t} \nonumber\\
& & - \upeta^2 N_t X^2\cdot \|\hve{\mu}_{y} -
\hve{\mu}_{t}\|_2^2 -\frac{\upeta N_t
    \max\left\{\alpha_t,\frac{N_t}{n_t} -1 \right\}  X p^*_t}{n_t}\cdot\|\hve{\mu}_{y} -
          \hve{\mu}_{t}\|_2 \nonumber\\
& \geq & \upeta \cdot \|\hve{\mu}_{y} -
\hve{\mu}_{t}\|_2^2 -\frac{2\max\left\{\alpha_t, \beta_t,
         \frac{N_t}{n_t} -1, \frac{N_{t-1}}{n_t} -1\right\}
         p^*_t}{n_t} - \upeta^2 N_t X^2 \cdot \|\hve{\mu}_{y} -
\hve{\mu}_{t}\|_2^2 \nonumber\\
& & - \frac{\upeta N_t \max\left\{\alpha_t,\frac{N_t}{n_t} -1 \right\}
    X p^*_t}{n_t}\cdot\|\hve{\mu}_{y} -
          \hve{\mu}_{t}\|_2 \nonumber\\
& & = \upeta \cdot \|\hve{\mu}_{y} -
\hve{\mu}_{t}\|_2^2 -\frac{2\max\left\{\alpha_t, \beta_t, \frac{N_t}{n_t} -1, \frac{N_{t-1}}{n_t} -1\right\}p^*_t}{n_t}  - \upeta^2 N_t X^2\cdot \|\hve{\mu}_{y} -
\hve{\mu}_{t}\|_2^2 \nonumber\\
& & - \frac{\upeta N_t \max\left\{n_t \alpha_t , N_t -n_t \right\} X p^*_t}{n_t^2}\cdot\|\hve{\mu}_{y} -
          \hve{\mu}_{t}\|_2 \nonumber\\
& = & \tilde{\upeta} \cdot \|\hve{\mu}_{y} -
\hve{\mu}_{t}\|_2 -\frac{2\max\left\{\alpha_t, \beta_t,
      \frac{N_t}{n_t} -1, \frac{N_{t-1}}{n_t} -1\right\} p^*_t}{n_t}
      - \tilde{\upeta}^2 N_t X^2 \nonumber\\
& & - \frac{\tilde{\upeta} N_t \max\left\{n_t \alpha_t , N_t -n_t \right\} X p^*_t}{n^2}\nonumber\\
  &   \geq & \underbrace{-a \tilde{\upeta}^2 + b \tilde{\upeta} +
         c}_{\defeq J(\tilde{\upeta})} ,\label{binfF}
  \end{eqnarray}
  with $\tilde{\upeta} \defeq \upeta \cdot \|\hve{\mu}_{y} -
\hve{\mu}_{t}\|_2$ and:
  \begin{eqnarray}
a & \defeq & N_t X^2 ,\\
    b & \defeq &  \|\hve{\mu}_{y} -
\hve{\mu}_{t}\|_2 -  \varepsilon_t(1+\varepsilon_t)\cdot p^*_t X,\\
    c & \defeq & -\frac{2\varepsilon_t p^*_t}{n_t} ,
    \end{eqnarray}
where $\varepsilon_t$ is any real satisfying
\begin{eqnarray}
\varepsilon_t & \geq & \max\left\{\alpha_t, \beta_t, \frac{N_t}{n_t} -1, \frac{N_{t-1}}{n_t} -1\right\}.\label{condEPSILON2}
\end{eqnarray}
Remark that
\begin{eqnarray}
2\sqrt{a(1+\varepsilon_t)\cdot -c} & = & 2\sqrt{2\varepsilon_t}(1+\varepsilon_t) \sqrt{p^*_t} X\nonumber,
\end{eqnarray}
so if we can guarantee that $b^2 \geq 4a(1+\varepsilon_t) \cdot -c$, then
fixing $\tilde{\upeta} \defeq (1-\gamma_t) b / (2a)$ for some
$\gamma_t \in [0,1]$ yields from \eqref{binfF}
\begin{eqnarray}
J(\tilde{\upeta}) & = & \frac{b^2(1-\gamma_t^2)}{4a} + c\nonumber\\
& \geq & - \varepsilon_t c + \gamma^2_t (1+\varepsilon_t)c\label{binfJT}
\end{eqnarray}
The condition on $b$ is implied by the following one, since $p^*_t \leq 1$:
\begin{eqnarray}
\|\hve{\mu}_{y} -
\hve{\mu}_{t}\|_2 & \geq & 2\sqrt{2\varepsilon_t}(1+\varepsilon_t)
                           \sqrt{p^*_t} X+
                           \varepsilon_t(1+\varepsilon_t) \sqrt{p^*_t}  X\label{cond2211}.
\end{eqnarray}
Fix any $K_t > 1$. It is easy to check that for any 
\begin{eqnarray}
\varepsilon_t & \leq & \sqrt{K_t} - 1, \label{condEPSILON}
\end{eqnarray}
we have $\varepsilon_t \leq 2(\sqrt{K_t} - \sqrt{2})\sqrt{\varepsilon_t}$,
so a
         sufficient condition to get \eqref{cond2211} is
\begin{eqnarray}
\sqrt{\varepsilon_t}(1+\varepsilon_t) & \leq & \frac{\|\hve{\mu}_{y} -
\hve{\mu}_{t}\|_2 }{2\sqrt{K_t} \sqrt{p^*_t}  X}.
\end{eqnarray}
Letting $f(z) \defeq \sqrt{z}(1+z)$, it is not hard to check that if
we pick $z = \min\{\sqrt{K_t}-1, u^2/K_{t}\}$ then $f(z) \leq u$: indeed,
\begin{itemize}
\item if the $\min$ is $u^2/K_t$, implying $u\leq
  \sqrt{K_t(\sqrt{K_t}-1)}$, then $f(z)$ being increasing we observe
  $f(z) \leq f(u^2/K_t)\leq u$, which simplifies for the rightmost inequality
  into $u\leq \sqrt{K_t(\sqrt{K_t}-1)}$, which is our assumption;
\item if the $\min$ is $\sqrt{K_t}-1$, implying $u\geq
  \sqrt{K_t(\sqrt{K_t}-1)}$, then this time we directly get $f(z) =
  \sqrt{\sqrt{K_t}-1}(1+\sqrt{K_t}-1) = \sqrt{K_t(\sqrt{K_t}-1)}\leq
  u$, as claimed.
\end{itemize}
To summarize, if we pick
\begin{eqnarray}
\varepsilon_t & \defeq & \min\left\{\sqrt{K_t}-1,  \frac{\|\hve{\mu}_{y} -
\hve{\mu}_{t}\|^2_2 }{4K_t^2 p^*_t X^2}\right\},
\end{eqnarray}
then we check that our precondition \eqref{condEPSILON} holds and we
obtain from \eqref{binfF} and \eqref{binfJT},
\begin{eqnarray}
F_{t+1} - \left(\frac{N_{t-1}}{n_t}
  -1\right)\cdot\frac{{p^*_t}}{n} & \geq &\frac{2\varepsilon_t^2
                                         p^*_t}{n_t} - \frac{2\gamma^2_t
                                         \varepsilon_t
                                         (1+\varepsilon_t)
                                         p^*_t}{n_t}.\label{binfFJ}
\end{eqnarray}
Suppose $\gamma_t$ satisfies 
\begin{eqnarray}
(1+\varepsilon_t)\gamma^2_t & \leq & \frac{\varepsilon_t}{2}.\label{condGAMMAT}
\end{eqnarray}
In this case, we further lowerbound \eqref{binfFJ} as
\begin{eqnarray}
F_{t+1} - \left(\frac{N_{t-1}}{n_t}
  -1\right)\cdot\frac{{p^*_t}}{n} & \geq & \frac{\varepsilon_t^2
                                         p^*_t}{n_t}\nonumber\\
& & =   \frac{p^*_t}{n_t}\cdot \left( \min\left\{\sqrt{K_t}-1,  \frac{\|\hve{\mu}_{y} -
\hve{\mu}_{t}\|^2_2 }{4K_t^2 p^*_t X^2}\right\}\right)^2.\label{eqBINFFX}
\end{eqnarray}
To simplify this bound and make it more readable, suppose we fix a
lowerbound
\begin{eqnarray}
\frac{\|\hve{\mu}_{y} -
\hve{\mu}_{t}\|^2_2 }{4 p^*_t X^2} & \geq & \delta_t,\label{condBFIN3}
\end{eqnarray}
for some $\delta_t > 0$. Some simple calculation shows that if we pick 
\begin{eqnarray}
K_t & \defeq & \left(1 + \frac{\delta_t }{1+\delta_t }\right)^2,
\end{eqnarray}
then the $\min$ in \eqref{eqBINFFX} is achieved in $\sqrt{K_t}-1$,
which therefore guarantees
\begin{eqnarray}
F_{t+1} - \left(\frac{N_{t-1}}{n_t}
  -1\right)\cdot\frac{{p^*_t}}{n_t} & \geq & \frac{p^*_t \delta_t}{n_t(1+\delta_t)},\label{eqBINFF2X}
\end{eqnarray}
and therefore gives the choice $\varepsilon_t =
\delta_t/(1+\delta_t)$. The constraint on $\gamma_t$ from
\eqref{condGAMMAT} becomes
\begin{eqnarray}
\gamma_t & \leq & \sqrt{\frac{\delta_t}{2(2+\delta_t)}},\label{condGAMMAT2}
\end{eqnarray}
and it comes from \eqref{condEPSILON2} that $\alpha_t, \beta_t \leq
\delta_t /(1+\delta_t)$ and 
\begin{eqnarray}
\frac{N_t}{n_t}, \frac{N_{t-1}}{n_t} & \leq & 1 + \frac{\delta_t}{1+\delta_t},
\end{eqnarray}
as claimed. This ends the proof of Lemma \ref{lemf}, after having
remarked that the learning rate $\upeta$ is then fixed to be (from \eqref{binfF})
\begin{eqnarray}
\upeta & \defeq & \frac{\tilde{\upeta} }{\|\hve{\mu}_{y} -
\hve{\mu}_{t}\|_2}\nonumber\\
& = & \frac{1-\gamma_t}{2 \|\hve{\mu}_{y} -
\hve{\mu}_{t}\|_2 N_t X^2} \cdot \left( \|\hve{\mu}_{y} -
\hve{\mu}_{t}\|_2 - \frac{\delta_t(2+\delta_t)}{(1+\delta_t)^2} p^*_t
X \right)\nonumber\\
& = & \frac{1-\gamma_t}{2 N_t X^2} \cdot \left( 1-
      \frac{\delta_t(2+\delta_t)}{(1+\delta_t)^2} \cdot \frac{p^*_t
X}{\|\hve{\mu}_{y} -
\hve{\mu}_{t}\|_2} \right),
\end{eqnarray}
and it satisfies, because of \eqref{condBFIN3},
\begin{eqnarray}
\upeta & \geq & \frac{1-\gamma_t}{2 N_t X^2} \cdot \left( 1-
      \frac{\sqrt{\delta_t p^*_t}(2+\delta_t)}{2(1+\delta_t)^2} \right)\label{binfETA}
\end{eqnarray}
and since $\|\hve{\mu}_{y} -
\hve{\mu}_{t}\|_2 \leq 2X$, 
\begin{eqnarray}
\upeta & \leq & \frac{1-\gamma_t}{2 N_t X^2} \cdot \left( 1-
      \frac{\delta_t p^*_t (2+\delta_t)}{2(1+\delta_t)^2} \right)
\end{eqnarray}
(we note that \eqref{condBFIN3} implies $\delta_t p^*_t\leq 1$) This ends the proof of Lemma \ref{lemf}.
\end{proof}
We now show a lowerbound on $Q_{t+1}$ in Theorem \ref{thBTT}. 
\begin{lemma}\label{binfqn}
Suppose the setting of Lemma \ref{lemf} holds.
Then 
\begin{eqnarray}
Q_{t+1} & \geq & \frac{p^*_t \delta_t}{n_t(1+\delta_t)}.
\end{eqnarray}
\end{lemma}
\begin{proof}
Remind that it comes from Theorem \ref{thBTT}
\begin{eqnarray}
Q_{t+1} & \defeq & F_{t+1}  -                              \left(\frac{N_{t-1}}{n_t} -1\right)\cdot
                                                   \properlossplus{t+1}{t}(S,\ve{w}_{t+1}).\nonumber
\end{eqnarray}
We have using Lemma \ref{lemlip},
\begin{eqnarray}
\properlossplus{t+1}{t}(S,\ve{w}_{t+1}) & \defeq & \expect_{{S}}[D_{U_t^\star}(y
\| u_t(\ve{w}_{t+1}^\top\ve{x}))] \nonumber\\
& \leq & \frac{1}{2n_t} \cdot \expect_{{S}}[(y -  u_t(\ve{w}_{t+1}^\top\ve{x}))^2] \nonumber\\
& & = \frac{1}{2n_t} \cdot (\expect_{{S}}[y] -2 \expect_{{S}}[y
    u_t(\ve{w}_{t+1}^\top\ve{x})] + \expect_{{S}}[u_t(\ve{w}_{t+1}^\top\ve{x})^2] ) \nonumber\\
& \leq & \frac{\expect_{{S}}[y] + \expect_{{S}}[u_t(\ve{w}_{t+1}^\top\ve{x})]}{2n_t}\nonumber\\
& \leq & \frac{p^*_t}{n_t},
\end{eqnarray}
because $u_t(z)\leq 1$. We get
\begin{eqnarray}
\left(\frac{N_{t-1}}{n_t} -1\right) \cdot
                                                   \properlossplus{t+1}{t}(S,\ve{w}_{t+1})
  & \leq & \left(\frac{N_{t-1}}{n_t} -1\right)\cdot\frac{p^*_t}{n_t},
\end{eqnarray}
so using Lemma \ref{lemf}, we get
\begin{eqnarray}
Q_{t+1} & \geq & F_{t+1} - \left(\frac{N_{t-1}}{n_t} -1\right)\cdot\frac{p^*_t}{n_t} \nonumber\\
& \geq & \frac{p^*_t \delta_t}{n_t(1+\delta_t)},
\end{eqnarray}
as claimed.
\end{proof}
Remind from Theorem \ref{thBTT} that 
\begin{eqnarray}
\properlossplus{t+1}{t+1}(S,\ve{w}_{t+1}) & \leq & \properlossplus{t}{t}(S,\ve{w}_{t})
                                              - \expect_{{S}}[D_{U_t}(
        \ve{w}_{t}^\top\ve{x}\| u^{-1}_{t} \circ
                                              u_{t+1}(\ve{w}_{t+1}^\top\ve{x}))]
                                              - L_{t+1} - Q_{t+1},\nonumber
\end{eqnarray}
and we know that
\begin{itemize}
\item $\expect_{{S}}[D_{U_t}(
        \ve{w}_{t}^\top\ve{x}\| u^{-1}_{t} \circ
                                              u_{t+1}(\ve{w}_{t+1}^\top\ve{x}))]
                                              \geq 0$, because a
                                              Bregman divergence
                                              cannot be negative;
\item $L_{t+1}\geq 0$ from Lemma \ref{lir1};
\item $Q_{t+1} \geq p^*_t \delta_t/(n_t(1+\delta_t))$ from Lemma
  \ref{binfqn} (assuming the conditions of Lemma \ref{lemf}).
\end{itemize}
Putting this altogether, we get
\begin{eqnarray}
\properlossplus{t+1}{t+1}(S,\ve{w}_{t+1}) & \leq & \properlossplus{t}{t}(S,\ve{w}_{t})- \frac{p^*_t \delta_t}{n_t(1+\delta_t)},\nonumber
\end{eqnarray}
which then easily translates into the statement of Theorem
\ref{thCLU}.

\section{Proof of Corollary \ref{corCLU}}\label{proof_corCLU}

To make things explicit, we replace Step 3 in the \clu~by the
following new Step 3:

\begin{itemize}
\item [Step 3] fit $\hve{y}_{t+1}$ by solving for
  global optimum:
\begin{eqnarray}
\hve{y}_{t+1} & \defeq & \arg\min_{\hat{\ve{y}}}
\expect_{{S}}[D_{U^\star_{t}}(\hat{y}\|\hve{y}_{t})]\mbox{
                    \hspace{2cm}//proper composite fitting of
                         $\hve{y}_{t+1}$ given $\ve{w}_{t+1}, u_t$}\nonumber\\
& & \mbox{ s.t. }  \left\{\begin{array}{l}
\hat{y}_{i+1} - \hat{y}_{i}
\in [n_t\cdot (\ve{w}_{t+1}^\top(\ve{x}_{i+1} - \ve{x}_i)), N_t\cdot
(\ve{w}_{t+1}^\top(\ve{x}_{i+1} - \ve{x}_i))]\:\:, \forall i \in
                        [m-1]\\
                        \hat{y}_1 \in [(1-\beta_t)
u_{t}(\ve{w}_{t+1}^\top\ve{x}_1), (1+\alpha_t)
                            u_{t}(\ve{w}_{t+1}^\top\ve{x}_1)]\\
\hat{y}_m \leq 1
\end{array}\right. \label{pbregU2}.
\end{eqnarray}
\end{itemize}
The only step that needs update in the proof of Theorem \ref{thCLU} is
Lemma \ref{lir1}. We now show that the property still holds for this
new Step 3.
\begin{lemma}\label{lir2}
    $L_{t+1} \geq 0$, $\forall t$.
\end{lemma}
\begin{proof}
The proof proceeds from the same steps as for Lemma \ref{lir1}. We reuse
the same notations. This time, we get the
Lagrangian,
\begin{eqnarray}
{\mathcal{L}}(\hve{y},S|\ve{{\lambdaleft}}, \ve{{\lambdaright}}, \rholeft, \rhoright) & \defeq & \expect_{{\mathcal{S}}}[
D_{U^\star_{t}}(\hat{y}\|y_t)] + \sum_{i=1}^{m-1} {\lambdaleft}_i \cdot (\hat{y}_i
                                                                          - \hat{y}_{i+1} + n_{i+1} - n_{i}) \nonumber\\
  & & + \sum_{i=1}^{m-1}{\lambdaright}_i \cdot
(\hat{y}_{i+1} - \hat{y}_i - N_{i+1} + N_i) + \rholeft \cdot
      ((1-\beta_t) q_1 - \hat{y}_1) \nonumber\\
  & & + \rholeft' \cdot
(\hat{y}_1 - (1+\alpha_t) q_1) + \rhoright \cdot
(\hat{y}_m-1) \:\:,
\end{eqnarray}
and
the following KKT conditions for the optimum:
\begin{eqnarray}
\omega_i(u^{-1}_{t}(\hat{y}_i) - u^{-1}_{t}(\hat{y}_{ti})) + {\lambdaleft}_i -
{\lambdaleft}_{i-1} + {\lambdaright}_{i-1} - {\lambdaright}_{i} & = & 0\:\:, \forall i = 2, 3,
..., m-1\:\:,\label{eq12}\\
\omega_i (u^{-1}_{t}(\hat{y}_1) - u^{-1}_{t}(\hat{y}_{t1})) + {\lambdaleft}_{1} -
{\lambdaright}_{1} -
  \rholeft + \rholeft' & = & 0\:\:, \label{eq22}\\
\omega_i (u^{-1}_{t}(\hat{y}_m) - u^{-1}_{t}(\hat{y}_{tm})) - {\lambdaleft}_{m-1} +
{\lambdaright}_{m-1} - \rhoright& = & 0\:\:, \label{eq32}\\
\hat{y}_{i+1} - \hat{y}_i & \in & [n_{i+1} - n_i, N_{i+1} -
 N_i]\:\:, \forall i \in [m-1]\:\:,\\
\hat{y}_1 & \in &q_1 \cdot  [1-\beta_t, 1+\alpha_t ]\:\:,\\
{\lambdaleft}_i \cdot  (\hat{y}_i
- \hat{y}_{i+1} +n_{i+1} - n_{i})  & = & 0\:\:, \forall i \in [m-1]\:\:,\label{kktlambda2}\\
{\lambdaright}_i \cdot (\hat{y}_{i+1} - \hat{y}_i - N_{i+1} + N_i)  & = & 0\:\:,
\forall i \in [m-1]\:\:,\label{kktrho2}\\
\rholeft \cdot  ((1-\beta) q_1 - \hat{y}_1)  & = & 0\:\:, \label{kktpil2}\\
\rholeft' \cdot  (\hat{y}_1 - (1+\alpha) q_1)  & = &
                                                     0\:\:, \label{kktpir2}\\
\rhoright \cdot
(\hat{y}_m-1) & = & 0\:\:, \label{kktpm}\\
\ve{{\lambdaleft}}, \ve{{\lambdaright}} & \succeq &
                                                    \ve{0}\:\:,\nonumber\\
\rholeft, \rholeft', \rhoright & \geq & 0\:\:.
\end{eqnarray}
Letting again $\sigma_i \defeq \sum_{j=i}^m \omega_j  (u^{-1}_{t}(\hat{y}_{tj}) -
u^{-1}_{t}(\hat{y}_j))$ (for $i = 1, 2, ..., m$) and $\hat{y}_0$ and $q_0$ 
any identical reals, we obtain this time:
\begin{eqnarray}
\sigma_i & = & {\lambdaright}_{i-1} - {\lambdaleft}_{i-1} + \rhoright\:\:, \forall i \in \{2, 3,
..., m\}\:\:,\\
\sigma_1 & = &- \rholeft + \rholeft'  + \rhoright  \:\:.\label{defsig1}
\end{eqnarray}
We now remark that just like in \eqref{bfinRHO1}, we still get
\begin{eqnarray}
(\sigma_i - \rhoright)\cdot ((\hat{y}_i - q_i)-(\hat{y}_{(i-1)} -
  q_{(i-1)})) & \geq & 0, \forall i>1,\label{bfinRHO11}
\end{eqnarray}
since the expression of the corresponding $\sigma$s does not change. The proof
changes for $\sigma_1$ as this time,
\begin{eqnarray}
(\sigma_1 - \rhoright)\cdot ((\hat{y}_1 - q_1)-(\hat{y}_{0} -
  q_{0})) & = & (- \rholeft + \rholeft')\cdot (\hat{y}_1 - q_1),
\end{eqnarray}
and we have the following possibilities:
\begin{itemize}
\item suppose $\rholeft > 0$. In this case,
KKT condition (\ref{kktpil2}) implies
$\hat{y}_1 = (1-\beta_t) q_1$,
implying $\hat{y}_1 - q_1 = -\beta_t q_1 \leq 0$, and also $\hat{y}_1
\neq (1+\alpha_t) q_1$, implying from KKT condition \eqref{kktpir2}
$\rholeft' = 0$, which gives us $(- \rholeft + \rholeft')\cdot
(\hat{y}_1 - q_1) = - \rholeft \cdot (\hat{y}_1 - q_1) \geq 0$.
\item suppose $\rholeft' > 0$. In this case,
the KKT condition (\ref{kktpir2}) implies
$\hat{y}_1 = (1+\alpha) q_1$ and so $\hat{y}_1 -q_1 = \alpha q_1 \geq 0$, but also so $\hat{y}_1 \neq (1-\beta) q_1$, so
$\rholeft = 0$, which gives us $(- \rholeft + \rholeft')\cdot
(\hat{y}_1 - q_1) = \rholeft' \cdot (\hat{y}_1 - q_1) \geq 0$.
\item If both $\rholeft = \rholeft' = 0$, we note $ (- \rholeft +
  \rholeft')\cdot (\hat{y}_1 - q_1) = 0$,
\end{itemize}
and so \eqref{bfinRHO1} also holds for $i=1$, which allows us to
conclude in the same way as we did for Lemma \ref{lir1}, and ends the proof of Lemma \ref{lir2}.
\end{proof}

\section{Proof of Lemma \ref{lemCONC}}\label{proof_lemCONC}

\begin{figure}[t]
\begin{center}
\includegraphics[trim=20bp 520bp 550bp
140bp,clip,width=0.99\linewidth]{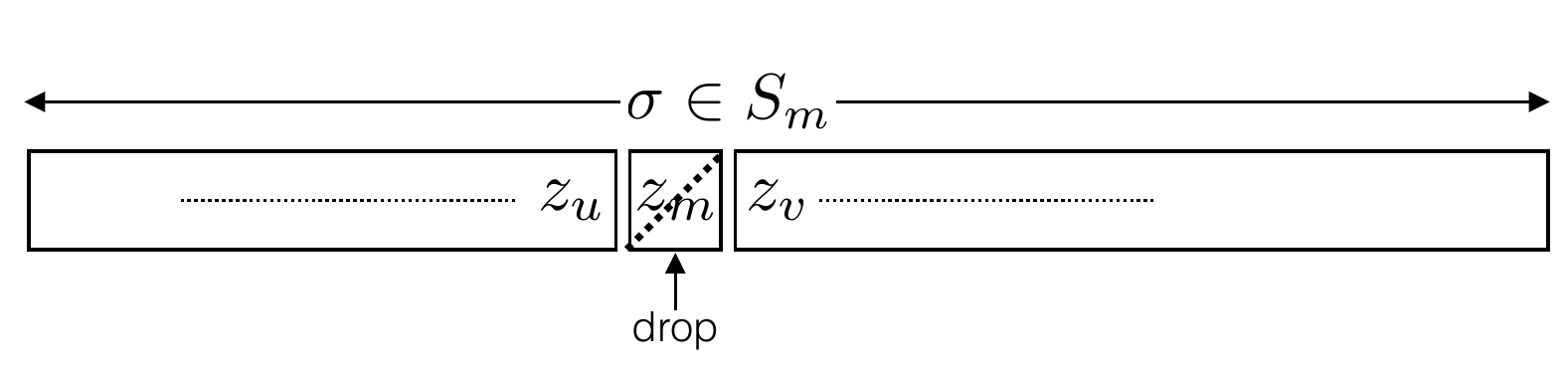}
\end{center}
\caption{Crafting from $\sigma \in S_m$ a subset of $m-1$ reals for which
  the induction hypothesis can be applied in the proof of Lemma \ref{lemCONC} (see text).}
  \label{f-perm}
\end{figure}

Let us drop the iteration index, thus letting $z_{i} \defeq z_{Ti}$
for $i=0, 1, ..., m+1$ (with $z_{0} \defeq
\zmin{T}$ and $z_{m+1} \defeq
\zmax{T}$). We thus have $z_{i} \leq z_{i+1}, \forall i$. We now pick one specific element in $\mathcal{U}(\ve{w},S)$, such that
\begin{eqnarray}
u(z_{i}) & = & (-
\cbr')^{-1}(z_{i}),\label{eqID1}
\end{eqnarray}
for $i\in [d]$, which complies with the definition of $\mathcal{U}$ as both $u$ and $(-
\cbr')^{-1}$ are non decreasing. We then have
\begin{eqnarray}
\int_{z_1}^{z_m} |(-
\cbr')^{-1}(z) - u(z)|\mathrm{d}z & = & \sum_{i=1}^{m-1} \int_{z_{i}}^{z_{i+1}} |(-
\cbr')^{-1}(z) - u(z)|\mathrm{d}z\nonumber\\
& \leq & \sum_{i=1}^{m-1}
                                           (u(z_{i+1}) -
                                           u(z_i))(z_{i+1} - z_i)\nonumber\\
& \leq & N \sum_{i=1}^{m-1} (z_{i+1} - z_i)^2,\label{bsup11}
\end{eqnarray}
where the first inequality holds because of \eqref{eqID1} and $u$ is non decreasing, and the second inequality holds because of the constraint in Step
3. Let $S_m \ni \sigma: [m] \rightarrow [m]$ be a
permutation of the indices. We now show
\begin{eqnarray}
\sum_{i=1}^{m-1} (z_{\sigma(i+1)} -
  z_{\sigma(i)})^2 & \geq & \sum_{i=1}^{m-1} (z_{i+1} - z_i)^2,
                            \forall m>1, \forall
                            \sigma \in S_m. \label{permS}
\end{eqnarray}
We show this by induction on $m$. The result is trivially true for
$m=2$. Considering any $m>2$ and any permutation $\sigma \in S_m$,
suppose the order of the $z$s in the permutation is as in Figure
\ref{f-perm}. Let $\Sigma_{\mathrm{tot}}\defeq \sum_{i=1}^{m-1} (z_{\sigma(i+1)} -
  z_{\sigma(i)})^2$, which therefore includes term $(z_m - z_u)^2 +
  (z_v - z_m)^2$. Now, drop $z_m$. This gives us a partial sum, $\Sigma_{\mathrm{partial}}$, over
  $\{z_1, z_2, ..., z_{m-1}\}$ described by a permutation $\sigma \in
  S_{m-1}$ for which the induction hypothesis applies. We then have
  two cases:\\
\noindent \textbf{Case 1}: $1<\sigma(m)<m$, which implies that $z_m$
is "inside" the ordering given by $\sigma$ and is in fact the case
depicted in Figure \ref{f-perm}. In this case and using notations from
Figure \ref{f-perm}, we get:
\begin{eqnarray}
\Sigma_{\mathrm{tot}} & = & \Sigma_{\mathrm{partial}} + (z_m - z_u)^2
                            +  (z_v - z_m)^2 - (z_v - z_u)^2,\label{lbound1}
\end{eqnarray}
and the induction hypothesis yields
\begin{eqnarray}
\Sigma_{\mathrm{partial}} & \geq & \sum_{i=1}^{m-2} (z_{i+1} - z_i)^2. \label{lbound2}
\end{eqnarray}
So to show \eqref{permS} we just need to show
\begin{eqnarray}
\underbrace{\sum_{i=1}^{m-2} (z_{i+1} - z_i)^2 + (z_m - z_u)^2
                            +  (z_v - z_m)^2 - (z_v -
  z_u)^2}_{\mbox{Lowerbound on $\Sigma_{\mathrm{tot}}$ from
  \eqref{lbound1} and \eqref{lbound2}}} & \geq & \sum_{i=1}^{m-1} (z_{i+1} - z_i)^2,
\end{eqnarray}
which equivalently gives
\begin{eqnarray}
(z_m - z_u)^2 + (z_m - z_v)^2 & \geq & (z_v - z_u)^2 +(z_m - z_{m-1})^2.\label{sEQ1}
\end{eqnarray}
After putting $(z_v - z_u)^2$ in the LHS and simplifying, we get
equivalently that the induction holds if 
\begin{eqnarray}
2z_m^2 -2z_mz_u -2z_mz_v +2z_vz_u & \geq & (z_m - z_{m-1})^2. \label{sEQ2}
\end{eqnarray}
The LHS factorizes conveniently as $2z_m^2 -2z_mz_u -2z_mz_v +2z_vz_u
= 2(z_m - z_u)(z_m - z_v)$. Since by hypothesis $z_1\leq z_2 ... \leq
z_{m-1}\leq z_m$, we get $2(z_m - z_u)(z_m - z_v) \geq
2(z_m-z_{m-1})^2$, which implies \eqref{sEQ2} holds and the induction
is proven.\\
\noindent \textbf{Case 2}: $\sigma(m)=m$ (the case $\sigma(m)=1$ give
the same proof). In this case, $z_m$ is at
the "right" of the permutation's ordering. Using notations from Figure
\ref{f-perm}, we get in lieu of \eqref{lbound1},
\begin{eqnarray}
\Sigma_{\mathrm{tot}} & = & \Sigma_{\mathrm{partial}} + (z_m - z_u)^2,\label{lbound22}
\end{eqnarray}
and leaves us with the following result to show:
\begin{eqnarray}
\sum_{i=1}^{m-2} (z_{i+1} - z_i)^2 + (z_m - z_u)^2 & \geq & \sum_{i=1}^{m-1} (z_{i+1} - z_i)^2,
\end{eqnarray}
which simplifies in $(z_m - z_u)^2 \geq (z_m - z_{m-1})^2$, which is
true by assumption ($z_u\leq z_{m-1} \leq z_m$).\\

To summarize, we have shown that $\forall \sigma: [m] \rightarrow
[m]$, 
\begin{eqnarray}
\int_{z_1}^{z_m} |(-
\cbr')^{-1}(z) - u(z)|\mathrm{d}z & \leq & \sum_{i=1}^{m-1} (z_{\sigma(i+1)} -
  z_{\sigma(i)})^2 .
\end{eqnarray}
Assuming the $\varepsilon$-NN graph is 2-vertex-connected, we square
the graph. Because of the triangle inequality on norm $\|.\|$, every
edge has now length at most $2\varepsilon$ and the graph is
Hamiltonian, a result known as Fleischner's Theorem
\citep{fTS}, \cite[p. 265, F17]{gyHO}. Consider any Hamiltonian path and the permutation $\ve{\sigma}$ of
$[m]$ it induces. We thus get $\|\ve{x}_{\sigma(i+1)} -
\ve{x}_{\sigma(i)}\|\leq 2\varepsilon, \forall i$, and so Cauchy-Schwarz
inequality yields:
\begin{eqnarray}
\sum_{i=1}^{m-1} (z_{\sigma(i+1)} -
  z_{\sigma(i)})^2 & \defeq & \sum_{i=1}^{m-1} (\ve{w}^\top\ve{x}_{\sigma(i+1)}
  -\ve{w}^\top\ve{x}_{\sigma(i)})^2 \nonumber\\
& \leq & \|\ve{w}\|_*^2 \sum_{i=1}^{m-1} \|\ve{x}_{\sigma(i+1)}
  -\ve{x}_{\sigma(i)}\|^2 \nonumber\\
& \leq & 2m\varepsilon^2 \cdot \|\ve{w}\|_*^2,\label{bsup22}
\end{eqnarray}
as claimed, where $\|.\|_*$ is the dual norm of $\|.\|$. We assemble \eqref{bsup11} and \eqref{bsup22} and get:
\begin{eqnarray*}
\int_{z_1}^{z_m} |(-
\cbr')^{-1}(z) - u(z)|\mathrm{d}z & \leq & 2 N m\varepsilon^2 \cdot \|\ve{w}\|_*^2, 
\end{eqnarray*}
which is the statement of the Lemma.\\

\noindent \textbf{Remark}: had we measured the $\ell_1$ discrepancy using the loss and not its link (and adding a second order differentiability condition), we could have used the fact that a Bregman divergence between two points is proportional to the square loss to get a result similar to the Lemma (see Section \ref{sec_FactSheetBreg}).

\end{document}